\documentclass{article} 
\usepackage{iclr2026_conference,times}

\usepackage{amsmath,amsfonts,bm}

\def\eqref#1{equation~\ref{#1}}

\def\1{\bm{1}}

\DeclareMathAlphabet{\mathsfit}{\encodingdefault}{\sfdefault}{m}{sl}
\SetMathAlphabet{\mathsfit}{bold}{\encodingdefault}{\sfdefault}{bx}{n}

\usepackage{wrapfig}
\usepackage{caption}
\usepackage{subcaption}
\usepackage{siunitx}
\usepackage[ruled,vlined]{algorithm2e}
\usepackage{caption}
\DontPrintSemicolon            
\SetKwInput{KwIn}{Input}       
\SetKwInput{KwOut}{Output}     
\SetKwComment{Comment}{$\triangleright$\ }{}  
\SetAlgoCaptionSeparator{.}    
\SetAlgoInsideSkip{smallskip}  
\SetAlCapSkip{0.6em}           
\SetKwProg{Fn}{Function}{:}{end} 

\SetAlFnt{\small}
\usepackage{hyperref}
\usepackage{url}
\usepackage{microtype}
\usepackage{graphicx}
\usepackage{pgfplots}
\usepackage{booktabs} 
\usepackage{multirow}
\usepackage{xspace}

\usepgfplotslibrary{groupplots}
\usetikzlibrary{patterns}
\pgfplotsset{compat=1.17} 
\pgfplotsset{bar width/.style={/pgf/bar width=#1}}
\usepackage{tikz}
\usetikzlibrary{arrows.meta,positioning,fit,backgrounds,calc,decorations.pathreplacing}
\usetikzlibrary{arrows.meta,calc}
  
\newcommand{\method}{\textsc{TCUQ}}%

\newtheorem{lemma}{Lemma}

\newcommand{\lemref}[1]{%
    \hyperref[#1]{Lemma~\ref*{#1}}%
}

\newtheorem{theorem}{Theorem}

\newcommand{\thmref}[1]{%
    \hyperref[#1]{Theorem~\ref*{#1}}%
}

\title{Drift-to-Action Controllers: Budgeted Interventions with Online Risk Certificates}

\author{Ismail Lamaakal$^{1\kern0.01em*}$,\ 
Chaymae Yahyati$^{1\kern0.01em*}$,\ 
Khalid El Makkaoui$^{1}$,\ 
Ibrahim Ouahbi$^{1}$,\ 
Yassine Maleh$^{2}$\\[4pt]
$^{1}$Multidisciplinary Faculty of Nador, Mohammed First University, Oujda 60000, Morocco\\
$^{2}$Laboratory LaSTI, ENSAK, Sultan Moulay Slimane University, Khouribga 54000, Morocco\\[4pt]
\texttt{\{ismail.lamaakal,\hspace{0.1em}chaymae.yahyati,\hspace{0.1em}k.elmakkaoui,\hspace{0.1em}i.ouahbi\}@ump.ac.ma}\\
\hspace*{\fill}\texttt{y.maleh@usms.ma}\hspace*{\fill}%
}

\iclrfinalcopy 
\begin{document}

\maketitle
\renewcommand{\method}{\textsc{Drift2Act}} 

\footnotetext[1]{* Equally contributed and led the project.}

\begin{abstract}
Deployed machine learning systems face distribution drift, yet most monitoring pipelines stop at alarms and leave the response underspecified under labeling, compute, and latency constraints. We introduce \emph{Drift2Act}, a drift-to-action controller that treats monitoring as constrained decision-making with explicit safety. \emph{Drift2Act} combines a sensing layer that maps unlabeled monitoring signals to a belief over drift types with an \emph{active risk certificate} that queries a small set of delayed labels from a recent window to produce an anytime-valid upper bound $U_t(\delta)$ on current risk. The certificate gates operation: if $U_t(\delta)\le\tau$, the controller selects low-cost actions (e.g., recalibration or test-time adaptation); if $U_t(\delta)>\tau$, it activates abstain/handoff and escalates to rollback or retraining under cooldowns. In a realistic streaming protocol with label delay and explicit intervention costs, \emph{Drift2Act}  achieves near-zero safety violations and fast recovery at moderate cost on WILDS Camelyon17, DomainNet, and a controlled synthetic drift stream, outperforming alarm-only monitoring, adapt-always adaptation, schedule-based retraining, selective prediction alone, and an ablation without certification. Overall, online risk certification enables reliable drift response and reframes monitoring as decision-making with safety.
\end{abstract}

\section{Introduction}
\label{sec:intro}

Modern machine learning models are increasingly deployed as long-lived services, where inputs evolve due to changing sensors, users, policies, environments, and feedback loops \citep{baylor2017tfx,breck2019data_validation,perdomo2020performative}. In such settings, distribution drift is the rule rather than the exception: a system that performs well at launch may degrade weeks later, and the degradation may be global (covariate shift), semantic (concept drift), or localized (subgroup-specific drift) \citep{quinonero2008dataset_shift,shimodaira2000covariate,widmer1996concept_drift,tsymbal2004concept_drift,raji2020accountability}. A large literature addresses \emph{catching} drift through statistical tests and representation monitoring \citep{gretton2012kernel,rabanser2019failing_loudly}, and a growing literature studies \emph{adapting} models via test-time adaptation and continual learning \citep{wang2021tent,kirkpatrick2017ewc}. However, production reliability requires a third ingredient: once evidence of drift appears, the system must decide \emph{what to do next} under real operational constraints \citep{baylor2017tfx,breck2019data_validation}.

Two gaps persist in current practice. First, drift detection often ends at an alarm \citep{rabanser2019failing_loudly}. Alarms are not actions: they do not specify whether to recalibrate, adapt, request labels, retrain, or roll back, nor do they account for budgets, latency, cooldowns, and governance constraints \citep{baylor2017tfx,breck2019data_validation,raji2020accountability}. Second, adaptation is frequently applied without verified safety \citep{wang2021tent}. Under delayed supervision, a system may silently operate in an unsafe regime, or may overreact and incur unnecessary cost. These gaps are amplified at scale, where heterogeneous streams and limited labeling capacity make it infeasible to continuously validate performance \citep{baylor2017tfx,breck2019data_validation}.

We propose a unified approach that treats \textbf{monitoring as decision-making with safety}. Our method introduces an \emph{active risk certificate} that produces an anytime-valid upper bound $U_t(\delta)$ on the current windowed risk $R_t$ using a small number of labels sampled from the live stream. The certificate functions as a safety layer: if $U_t(\delta)\le\tau$ the system can operate normally and apply inexpensive corrections, while if $U_t(\delta)>\tau$ the system triggers an explicit fallback (abstain/handoff) and escalates to stronger interventions such as rollback or retraining. To translate drift evidence into effective actions, we additionally maintain a belief state $b_t(d)=\mathbb{P}(D_t=d\mid z_{1:t})$ over drift types $d\in\{\mathrm{none},\mathrm{covariate},\mathrm{concept},\mathrm{subgroup}\}$ derived from representation, uncertainty, and calibration monitors. The controller then selects actions by maximizing a belief-weighted utility under budget and cooldown constraints.

Our evaluation uses a realistic streaming protocol with delayed labels, explicit intervention costs, and heavy-action cooldowns. Across WILDS Camelyon17, DomainNet, and a controlled synthetic drift stream, the proposed controller achieves near-zero safety violations at moderate cost, dominating the safety--cost frontier in Figure~\ref{fig:pareto_frontier} and recovering rapidly after drift events in Figure~\ref{fig:recovery_curves}. Table~\ref{tab:main_results} summarizes improvements in safety, recovery time, and worst-group robustness relative to alarm-only monitoring, adapt-always baselines, and schedule-driven retraining.

Our contributions are threefold: We introduce an active, anytime-valid risk certificate for drift monitoring under delayed supervision, we develop a belief-driven controller that maps drift evidence to cost-aware interventions under operational constraints, and we propose a streaming evaluation protocol that jointly measures safety, recovery, and operational cost.

\section{Related Work}
\label{sec:related_work}

\textbf{Drift detection and monitoring.}
Classical drift detection is closely related to sequential change-point testing, including CUSUM-style procedures and Page--Hinkley-type tests \citep{page1954continuous,hinkley1970inference}, and stream-specific detectors that adapt windows or monitor error-rate statistics online \citep{bifet2007adwin,gama2004ddm,gama2014survey}. Modern drift monitoring often uses representation-based two-sample testing, where changes are detected in an embedding space rather than raw inputs, including kernel Maximum Mean Discrepancy (MMD) tests \citep{gretton2012kernel} and learned test statistics \citep{howard2021timeuniform}. Another widely used paradigm is \emph{classifier two-sample testing}, which trains a discriminator to separate reference vs.\ recent samples and uses its accuracy/AUC as a shift signal \citep{lopezpaz2016revisiting}. Empirical studies highlight that dataset-shift detectors can fail silently or over-trigger depending on the shift type and monitoring choice \citep{rabanser2019failing_loudly}. Related work also studies specific shift structures such as label shift, with black-box estimation and detection procedures \citep{lipton2018labelshift}. In our paper, these monitoring signals are not endpoints; they are inputs to a belief state that drives interventions.

\textbf{Adaptation under distribution shift.}
A large literature updates models online under shift using unlabeled or sparsely labeled deployment data, including test-time adaptation (TTA) and self-supervised test-time training \citep{wang2021tent,sun2020ttt}. These methods can be effective for covariate and representation drift, but may be unstable under semantic changes or feedback loops, motivating safety-aware gating. Continual learning methods address non-stationarity by preventing catastrophic forgetting, using regularization, replay, or exemplar-based updates \citep{kirkpatrick2017ewc,lopezpaz2017gem,rebuffi2017icarl,li2016lwf}. Our controller treats adaptation as one action among several (including label acquisition, rollback, and retraining) and uses online risk certification to avoid prolonged unverified operation.

\textbf{Selective prediction and uncertainty.}
Selective prediction (abstention) reduces harm by deferring uncertain examples, with theoretical and empirical work on risk--coverage trade-offs and learned selection modules \citep{geifman2017selective,geifman2019selectivenet}. Uncertainty estimation tools that commonly support abstention and monitoring include deep ensembles and approximate Bayesian inference via dropout \citep{lakshminarayanan2017deepensembles,gal2016dropout}, as well as confidence and calibration diagnostics \citep{guo2017calibration}. We include abstain/handoff as a safe fallback, but trigger it using a \emph{certified} upper bound on current risk rather than confidence alone.

\textbf{Conformal inference and risk control.}
Conformal prediction provides distribution-free predictive inference under exchangeability \citep{vovk2005algorithmic}, and recent work generalizes conformal ideas to control broader risk functionals beyond coverage \citep{angelopoulos2022crc}. In parallel, confidence sequences provide anytime-valid bounds for streaming means under bounded observations and optional stopping \citep{howard2021timeuniform}. Our active risk certificate builds on these anytime-valid ideas to upper-bound windowed deployment risk using a small, randomly sampled labeled subset, and then integrates the bound directly into the control loop.

\textbf{ML systems, governance, and monitoring at scale.}
Production ML monitoring emphasizes logging, alerting, and remediation under operational constraints, while documenting and auditing model behavior via standardized reporting artifacts \citep{sculley2015debt,breck2017mltestscore,mitchell2019modelcards,gebru2021datasheets}. Benchmarks for distribution shift help evaluate monitoring and adaptation realistically, including WILDS \citep{sagawa2022wilds} and DomainNet streams commonly instantiated via multi-source domain adaptation benchmarks. Our framework aligns with this systems view by producing auditable certification signals and explicit actions with costs/cooldowns, evaluated under a streaming protocol that reports safety violations and worst-group robustness.

In summary, prior work addresses detection, adaptation, or abstention in isolation. We unify these components by framing drift monitoring as a constrained decision problem with an anytime-valid safety certificate that gates action selection.

\section{Sensing Layer: Monitoring Signals and Belief over Drift Type}
\label{sec:sensing_belief}

The sensing layer summarizes recent deployment behavior into a compact evidence vector that can be computed with little or no immediate supervision. This evidence is then mapped into a belief over drift types, enabling downstream actions that depend on whether the system is facing covariate shift, concept shift, or localized subgroup drift.

\subsection{Monitoring windows and representations}
\label{sec:windows_repr}

At each time step $t$, the system maintains a recent input window $\mathcal{B}_t=\{x_{t-n+1},\dots,x_t\}$ and a reference window $\mathcal{B}_{\mathrm{ref}}$ that represents nominal conditions. We monitor drift in a representation space induced by the deployed model through an embedding function $g_{\theta}:\mathcal{X}\rightarrow\mathbb{R}^{p}$:
\begin{equation}
r \;=\; g_{\theta}(x).
\end{equation}
In this equation, $x\in\mathcal{X}$ is an input, $\theta$ are the current model parameters, $g_{\theta}(\cdot)$ is the feature extractor, and $r\in\mathbb{R}^{p}$ is the embedding used by the monitors. We denote the sets of embeddings from the recent and reference windows by
\begin{equation}
\mathcal{R}_t \;=\; \{g_{\theta}(x):x\in\mathcal{B}_t\},
\qquad
\mathcal{R}_{\mathrm{ref}} \;=\; \{g_{\theta}(x):x\in\mathcal{B}_{\mathrm{ref}}\}.
\end{equation}
Here $\mathcal{R}_t$ and $\mathcal{R}_{\mathrm{ref}}$ are the monitored embedding collections for the recent and reference windows, respectively.

\subsection{Evidence signals}
\label{sec:evidence_signals}

We compute complementary drift signals that capture representation shift, uncertainty drift, and calibration degradation. The first signal is a kernel two-sample statistic (MMD) computed between $\mathcal{R}_t$ and $\mathcal{R}_{\mathrm{ref}}$:
\begin{equation}
\mathrm{MMD}^2_t \;=\; \left\|\ \frac{1}{|\mathcal{R}_t|}\sum_{r\in\mathcal{R}_t}\phi(r)\;-\;\frac{1}{|\mathcal{R}_{\mathrm{ref}}|}\sum_{r'\in\mathcal{R}_{\mathrm{ref}}}\phi(r')\ \right\|^2.
\end{equation}
In this equation, $\phi(\cdot)$ is the feature map of a positive definite kernel, $|\mathcal{R}_t|$ and $|\mathcal{R}_{\mathrm{ref}}|$ are window sizes, and $\mathrm{MMD}^2_t$ increases when the embedding distribution departs from nominal conditions.

The second signal is an uncertainty drift score based on predictive entropy. The model produces class probabilities $p_{\theta}(y\mid x)$, and we define the entropy for an input $x$ as
\begin{equation}
H_{\theta}(x) \;=\; -\sum_{y\in\mathcal{Y}} p_{\theta}(y\mid x)\,\log p_{\theta}(y\mid x).
\end{equation}
Here $\mathcal{Y}$ is the label set and $H_{\theta}(x)$ is higher when the model is less confident. We compare mean entropy in the recent and reference windows:
\begin{equation}
\Delta H_t \;=\; \frac{1}{|\mathcal{B}_t|}\sum_{x\in\mathcal{B}_t}H_{\theta}(x)\;-\;\frac{1}{|\mathcal{B}_{\mathrm{ref}}|}\sum_{x\in\mathcal{B}_{\mathrm{ref}}}H_{\theta}(x).
\end{equation}
In this equation, $\Delta H_t$ captures shifts in uncertainty that often correlate with out-of-distribution inputs or changes in class-conditional structure.

When delayed labels are available, we also monitor calibration drift using a streaming ECE proxy computed on the labeled set available at time $t$, denoted $\mathcal{S}_t$ (defined in Appendix~\ref{sec:problem_setup}). Let $q_{\theta}(x)=\max_{y}p_{\theta}(y\mid x)$ be the model confidence and let $\{I_b\}_{b=1}^{B}$ be confidence bins. The ECE estimate is
\begin{equation}
\mathrm{ECE}_t \;=\; \sum_{b=1}^{B}\frac{|\mathcal{S}_{t,b}|}{|\mathcal{S}_t|}\,\big|\mathrm{acc}_b(t)-\mathrm{conf}_b(t)\big|.
\end{equation}
In this equation, $\mathcal{S}_{t,b}=\{(x,y)\in\mathcal{S}_t:q_{\theta}(x)\in I_b\}$ is the labeled subset in bin $b$, $\mathrm{acc}_b(t)$ is the empirical accuracy in that bin, and $\mathrm{conf}_b(t)$ is the empirical mean confidence in that bin. We use $\Delta \mathrm{ECE}_t=\mathrm{ECE}_t-\mathrm{ECE}_{\mathrm{ref}}$ as a drift signal, where $\mathrm{ECE}_{\mathrm{ref}}$ is computed under nominal conditions.

We concatenate these monitors into a single evidence vector
\begin{equation}
z_t \;=\; \big[\ \mathrm{MMD}^2_t,\ \Delta H_t,\ \Delta \mathrm{ECE}_t\ \big]^{\top}\ \in\ \mathbb{R}^{m}.
\end{equation}
In this equation, $m$ is the number of active monitors.

\subsection{Belief over drift type}
\label{sec:belief}

We model drift type as a latent discrete state $D_t$ taking values in
\begin{equation}
\mathcal{D} \;=\; \{\mathrm{none},\ \mathrm{covariate},\ \mathrm{concept},\ \mathrm{subgroup}\}.
\label{eq8}
\end{equation}

\begin{wrapfigure}[14]{r}{0.46\linewidth}
\vspace{-3mm}
\captionsetup{font=footnotesize,skip=2pt,width=\linewidth}
\centering
\includegraphics[width=\linewidth]{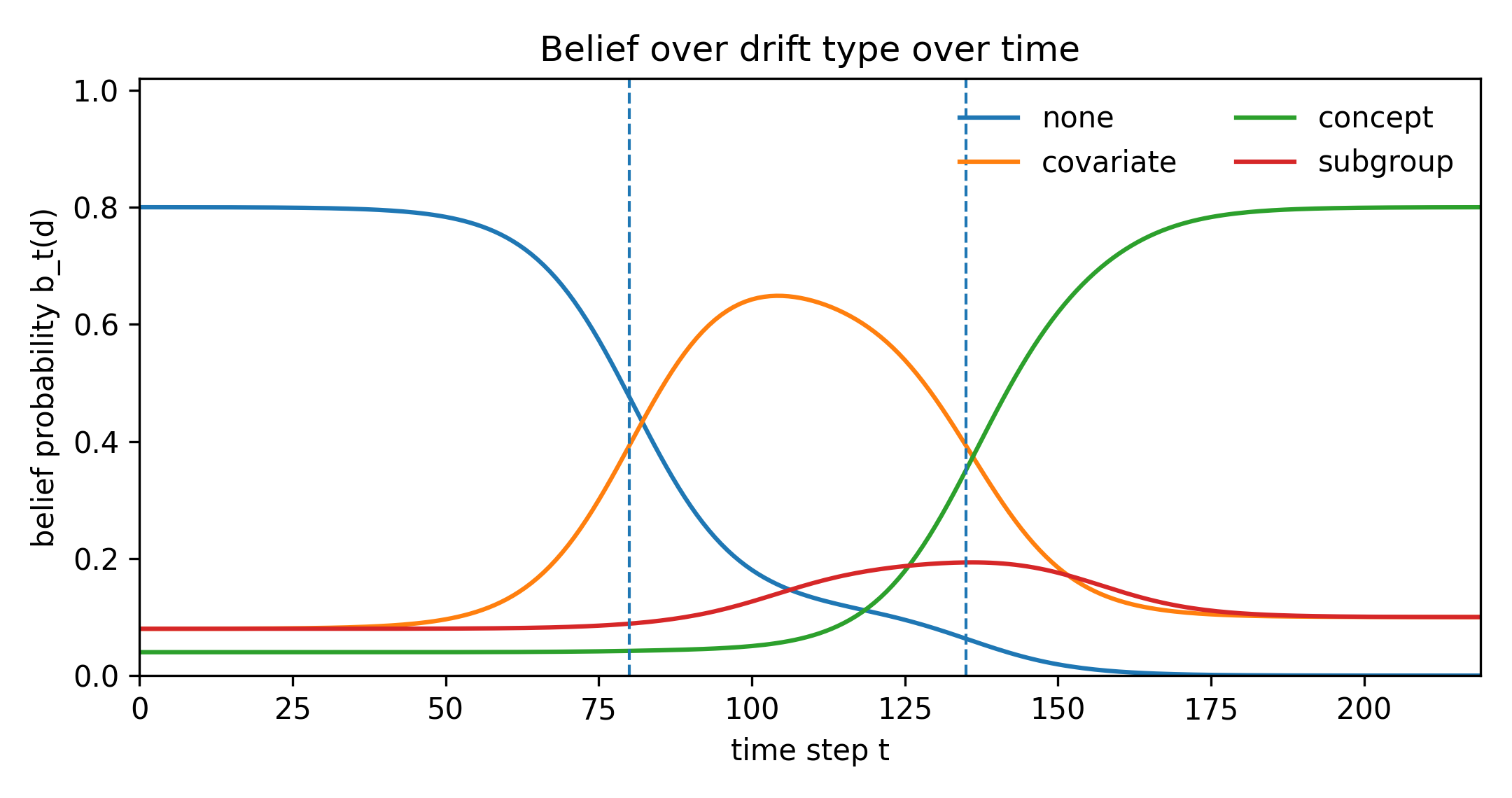}
\caption{\textbf{Belief evolution over drift types.} Posterior $b_t(d)=\mathbb{P}(D_t=d\mid z_{1:t})$ over $\{\mathrm{none},\mathrm{covariate},\mathrm{concept},\mathrm{subgroup}\}$ across time; dashed lines mark drift events that shift belief and guide intervention choice.}
\label{fig:belief_evolution}
\vspace{-4mm}
\end{wrapfigure}

The belief state is the posterior probability of each drift type given evidence up to time $t$:
\begin{equation}
b_t(d) \;=\; \mathbb{P}\!\left(D_t=d \mid z_{1:t}\right),
\qquad d\in\mathcal{D}.
\end{equation}
In this equation, $z_{1:t}=\{z_1,\dots,z_t\}$ is the evidence history and $b_t(d)$ is the probability assigned to type $d$.

We compute $b_t(d)$ using a lightweight Markov belief update with a transition model and an evidence likelihood. Let $T_{d'd}=\mathbb{P}(D_t=d\mid D_{t-1}=d')$ be the transition matrix and let $p(z_t\mid D_t=d)$ be the likelihood of observing evidence $z_t$ under drift type $d$. The update is
\begin{equation}
\tilde{b}_t(d) \;=\; \sum_{d'\in\mathcal{D}} T_{d'd}\,b_{t-1}(d'),
\qquad
b_t(d) \;=\; \frac{p(z_t\mid D_t=d)\,\tilde{b}_t(d)}{\sum_{d''\in\mathcal{D}} p(z_t\mid D_t=d'')\,\tilde{b}_t(d'')}.
\end{equation}
In these equations, $\tilde{b}_t(d)$ is the predicted belief before incorporating the new evidence, the numerator multiplies the predicted belief by the likelihood of the observed evidence under type $d$, and the denominator normalizes the distribution so that $\sum_{d\in\mathcal{D}} b_t(d)=1$. We fit $p(z_t\mid D_t=d)$ on synthetic drift episodes constructed from the base dataset, enabling drift-type inference directly from monitoring signals.

\section{Active Risk Certificate: Online Safety Layer}
\label{sec:active_risk_certificate}

This section introduces an active risk certificate that provides an anytime-valid upper bound on current deployment risk using only a small number of labels sampled from the live stream. The certificate acts as a safety layer that gates operational decisions: when the bound indicates potential violation of a target risk threshold, the system triggers a safe fallback and escalates interventions (see Appendix \ref{sec:algorithm} for more details).

\subsection{Windowed risk under delayed supervision}
\label{sec:window_risk}

At time $t$, we define a recent evaluation window that represents the current deployment regime:
\begin{equation}
W_t \;=\; \{t-N+1,\dots,t\},
\end{equation}
where $N$ is the window length and $W_t$ indexes the most recent $N$ time steps. The (windowed) deployment risk is the mean loss over examples in $W_t$:
\begin{equation}
R_t \;=\; \frac{1}{N}\sum_{i\in W_t} \ell_i,
\qquad
\ell_i \;=\; \ell\!\big(\hat{y}_{\theta}(x_i),y_i\big).
\end{equation}
In these equations, $\ell_i\in[0,1]$ is the per-example loss at index $i$ and $\ell(\cdot,\cdot)$ is a bounded task loss (for classification, $\ell_i=\mathbb{I}(\hat{y}_{\theta}(x_i)\neq y_i)$). The quantity $R_t$ captures the current error/risk level in the most recent regime. Because labels may be delayed, the system cannot evaluate all $\ell_i$ in $W_t$; it instead adaptively requests labels for a subset of indices.

\subsection{Active label sampling within the window}
\label{sec:active_sampling}

Let $Q_t \subseteq W_t$ denote the set of indices in the current window for which labels are requested and observed by time $t$ (or by a short fixed evaluation delay). The controller determines the sample size $|Q_t|=n_t$ based on drift evidence and budget constraints. For each queried index $i\in Q_t$, the system obtains $y_i$ and computes the realized loss $\ell_i$. The empirical risk on queried points is
\begin{equation}
\widehat{R}_t \;=\; \frac{1}{n_t}\sum_{i\in Q_t}\ell_i.
\end{equation}
In this equation, $\widehat{R}_t$ is a sample mean of losses computed from the queried subset, $n_t$ is the number of queried labels, and $Q_t$ is selected by uniformly sampling indices from $W_t$ (or by uniform sampling within each slice, when slice guarantees are desired). Uniform sampling is the key condition that allows finite-sample validity of the certificate with minimal assumptions about drift.

\subsection{Anytime-valid upper confidence bound}
\label{sec:anytime_bound}

We construct an anytime-valid upper bound $U_t(\delta)$ on the true window risk $R_t$ from the queried losses. Let $\delta\in(0,1)$ be a target failure probability. We define a nonnegative confidence radius $\mathrm{rad}(n_t,\delta)$ such that
\begin{equation}
\mathbb{P}\!\left(\forall t\in\{1,\dots,T\}:\ R_t \le \widehat{R}_t + \mathrm{rad}(n_t,\delta)\right) \;\ge\; 1-\delta,
\end{equation}
and set the certificate as
\begin{equation}
U_t(\delta) \;=\; \widehat{R}_t + \mathrm{rad}(n_t,\delta).
\end{equation}
In these equations, $\widehat{R}_t$ is the empirical risk on the queried subset, $\mathrm{rad}(n_t,\delta)$ is an anytime-valid uncertainty term that shrinks as more labels are queried, and $U_t(\delta)$ is the certified upper bound on current risk. We instantiate $\mathrm{rad}(\cdot,\cdot)$ using a confidence-sequence construction for bounded random variables, ensuring validity under optional stopping and adaptive query schedules. The explicit form and proof are provided in Appendix~\ref{app:confidence_sequence}.

\subsection{Safety gating and escalation}
\label{sec:safety_gating}

The certificate gates operation with respect to a target risk threshold $\tau\in(0,1)$:
\begin{equation}
\text{Safe operation at time $t$} \iff U_t(\delta)\le\tau.
\end{equation}
In this decision rule, $U_t(\delta)\le\tau$ certifies that current risk is below the target threshold, enabling normal operation and inexpensive interventions (e.g., recalibration). When $U_t(\delta)>\tau$, the system treats the regime as potentially unsafe and triggers a fallback policy (e.g., abstain/handoff) while escalating to stronger actions such as label acquisition, rollback, or retraining:
\begin{equation}
U_t(\delta)>\tau \ \Rightarrow\ \text{activate fallback and schedule escalation.}
\end{equation}
This gating mechanism separates \emph{monitoring evidence} (which suggests risk may be rising) from \emph{certified safety} (which bounds the risk using labeled verification).

\subsection{Label efficiency through targeted verification}
\label{sec:label_efficiency}

The controller uses drift evidence $z_t$ and the drift-type belief $b_t$ (Section~\ref{sec:sensing_belief}) to decide when to query labels for certification. Specifically, it increases $n_t$ when monitors indicate rising drift or when the certificate uncertainty $\mathrm{rad}(n_t,\delta)$ is large relative to the safety margin $\tau-\widehat{R}_t$, and it decreases $n_t$ when evidence is stable and the certificate is tight. This creates an explicit coupling between ``Catch'' and ``Operate'': labels are treated as a limited operational resource that is spent primarily when needed to maintain certified safety.

\subsection{Safety guarantee}
\label{sec:safety_theorem}

\textbf{Theorem 1 (Certified safe operation under random window sampling).}
Assume that at each time $t$ the queried set $Q_t$ is obtained by uniform random sampling from the window $W_t$ (or by uniform sampling within pre-defined strata) and that the loss satisfies $\ell_i\in[0,1]$. Then, for any $\delta\in(0,1)$, the certificate in Eq.~(\ref{eq18}) satisfies
\begin{equation}
\mathbb{P}\!\left(\forall t\in\{1,\dots,T\}:\ R_t \le U_t(\delta)\right) \;\ge\; 1-\delta.
\label{eq18}
\end{equation}
Consequently, with probability at least $1-\delta$, the system never remains in an unsafe regime without the certificate indicating a violation: whenever $R_t>\tau$, it holds that $U_t(\delta)>\tau$, and the gating rule activates fallback and escalation. A proof and the concrete confidence-sequence construction are given in Appendix~\ref{app:confidence_sequence}.
\vspace{-0.3cm}
\section{Drift-to-Action Controller}
\label{sec:controller}

This section converts drift evidence and the active risk certificate into concrete operational interventions. The controller is designed to be deployable: it selects from a small action set, respects budgets and cooldowns, and escalates only when safety is threatened.

\subsection{Decision variables and objective}
\label{sec:controller_objective}

At each time step $t$, the controller observes the current evidence vector $z_t$ (Section~\ref{sec:sensing_belief}), the belief state $b_t(\cdot)$ over drift types, and the current safety certificate $U_t(\delta)$ (Section~\ref{sec:active_risk_certificate}). It then selects an action $a_t\in\mathcal{A}$ (Table~\ref{tab:action_space_costs}) and, when applicable, a label request size $k_t$ for the label-query action.

We express the controller goal as minimizing expected cumulative task loss plus operational cost, while maintaining certified safety whenever the system is allowed to predict:
\begin{equation}
\min_{\pi}\ \mathbb{E}\!\left[\sum_{t=1}^{T} \ell_t \;+\; \lambda\, c(a_t)\right]
\qquad
\text{s.t.}\qquad
U_t(\delta)\le \tau\ \ \text{or fallback is engaged}.
\end{equation}
In this optimization, $\pi$ denotes a policy mapping the controller's observations to actions, $\ell_t=\ell(\hat{y}_{\theta}(x_t),y_t)$ is the per-step task loss, $c(a_t)$ is the per-step operational cost from Section~\ref{sec:costs}, and $\lambda\ge 0$ trades off predictive performance against operational burden. The constraint states that the system must either be certified safe, meaning $U_t(\delta)\le\tau$, or operate under an explicit fallback mode (abstain/handoff) that prevents unverified predictions.

\subsection{Operational constraints: budgets and cooldowns}
\label{sec:controller_constraints}

Real systems cannot retrain or roll back arbitrarily often. We model two classes of constraints. First, a labeling budget limits the total number of requested labels:
\begin{equation}
\sum_{t=1}^{T} k_t \;\le\; B_{\mathrm{lab}},
\end{equation}
where $k_t$ is the number of labels requested at time $t$ and $B_{\mathrm{lab}}$ is a fixed labeling budget. Second, heavy interventions have cooldowns. Let $\Delta_{\mathrm{rt}}$ and $\Delta_{\mathrm{rb}}$ be cooldown durations for retraining and rollback. If $t_{\mathrm{last}}^{\mathrm{rt}}$ is the most recent time a retrain was triggered, the retrain feasibility constraint is
\begin{equation}
a_t = A_4 \ \Rightarrow\ t - t_{\mathrm{last}}^{\mathrm{rt}} \ge \Delta_{\mathrm{rt}},
\end{equation}
and analogously for rollback with $A_5$ and $\Delta_{\mathrm{rb}}$. In these constraints, $\Delta_{\mathrm{rt}}$ and $\Delta_{\mathrm{rb}}$ capture engineering and infrastructure limitations such as retraining time, deployment approvals, and rollback validation.

\subsection{Receding-horizon utility under safety gating}
\label{sec:receding_horizon}

Rather than solving a full sequential decision problem, we use a receding-horizon utility that is simple, robust, and practical to implement. For each candidate action $a\in\mathcal{A}$, we compute a one-step utility that combines predicted risk reduction and operational cost:
\begin{equation}
\mathcal{U}_t(a) \;=\; \Delta_t(a)\;-\;\lambda\,c(a)\;-\;\gamma\,\mathrm{Viol}_t(a).
\end{equation}
In this equation, $\Delta_t(a)$ is the predicted improvement from taking action $a$ (larger is better), $c(a)$ is the operational cost, and $\mathrm{Viol}_t(a)$ is a penalty that discourages actions leading to uncertified operation. The weight $\gamma>0$ enforces safety preferences. The controller selects
\begin{equation}
a_t \;=\; \arg\max_{a\in\mathcal{A}_{t}^{\mathrm{feas}}}\ \mathcal{U}_t(a),
\end{equation}
where $\mathcal{A}_{t}^{\mathrm{feas}}\subseteq\mathcal{A}$ is the set of feasible actions at time $t$ after enforcing budget and cooldown constraints.

We define $\Delta_t(a)$ using the drift-type belief $b_t(d)$ to encode that different drifts benefit from different interventions:
\begin{equation}
\Delta_t(a) \;=\; \sum_{d\in\mathcal{D}} b_t(d)\, G(d,a),
\end{equation}
where $G(d,a)$ is a gain table specifying the expected benefit of action $a$ under drift type $d$. In this definition, $\mathcal{D}=\{\mathrm{none},\mathrm{covariate},\mathrm{concept},\mathrm{subgroup}\}$ is the drift-type set, $b_t(d)$ is the posterior probability of drift type $d$, and $G(d,a)$ can be learned from historical rollouts or calibrated via synthetic drift simulations. This design yields a controller that is interpretable and easily extensible: gains can be updated without changing the certificate.

\subsection{Safety-first escalation logic}
\label{sec:escalation_logic}

The certificate $U_t(\delta)$ gates the action space. When the certificate indicates potential violation, the controller prioritizes immediate safety via fallback and then schedules stronger interventions:
\begin{equation}
U_t(\delta)>\tau \ \Rightarrow\ a_t = A_6 \ \ \text{and}\ \ a_{t'}\in\{A_5,A_4\}\ \text{is scheduled when feasible}.
\label{eq25}
\end{equation}
In this rule, $A_6$ is abstain/handoff, which prevents unverified predictions, and $\{A_5,A_4\}$ represent rollback and retraining, which restore or improve performance. When the certificate is safe, the controller selects low-cost corrective actions that match the dominant drift type in $b_t$:
\begin{equation}
U_t(\delta)\le\tau \ \Rightarrow\ a_t \in \{A_0,A_1,A_2,A_3\}\ \text{chosen by Eq.~(\ref{eq25})}.
\end{equation}
This separation ensures that safety violations trigger immediate protective behavior, while normal operation remains cost-efficient.

\vspace{-0.3cm}
\section{Benchmarks, baselines, and main results}
\label{sec:main_results}

We evaluate on two real drift benchmarks and one controlled synthetic drift stream. The first benchmark is \textbf{WILDS Camelyon17} \citep{sagawa2022wilds,bandi2019camelyon17}, where environments correspond to hospitals and distribution shift arises from site-specific acquisition artifacts. The stream orders hospitals to induce multiple drift transitions, including a return to a previously seen hospital to test whether controllers avoid irreversible overreaction. The second benchmark is \textbf{DomainNet} \citep{peng2019domainnet}, which exhibits large domain gaps across sketch, clipart, painting, real, and quickdraw; the stream orders domains to induce heavy representation and covariate drift. The synthetic benchmark, \textbf{SyntheticDrift-CIFAR}, composes sudden covariate drift (corruption severity) \citep{hendrycks2019cifar10c}, gradual concept drift (class-conditional remapping ramped by $\alpha_t$) \citep{gama2014survey,tsymbal2004concept_drift}, and subgroup drift (a minority slice receives a style transform) \citep{gebru2021datasheets,sagawa2020groupdro}, enabling controlled tests of safety and recovery.

All methods use the same monitoring signals and belief model, and differ only in how they act. The \textbf{Alarm-only} baseline raises alarms when evidence exceeds a threshold but performs no interventions \citep{rabanser2019failing_loudly}. The \textbf{Adapt-always} baseline performs test-time adaptation continuously \citep{wang2021tent}. The \textbf{Retrain-on-schedule} baseline retrains periodically and ignores drift evidence, reflecting common production heuristics \citep{baylor2017tfx}. The \textbf{Selective prediction only} baseline abstains based on uncertainty but does not adapt or retrain \citep{geifman2017selective,lakshminarayanan2017deepensembles}. The \textbf{Controller (no certificate)} ablation selects actions using belief-weighted gains but does not gate operation using the certified bound $U_t(\delta)$. The proposed \textbf{Controller + certificate} uses full safety gating and active risk verification \citep{howard2021timeuniform}.

Table~\ref{tab:main_results} summarizes the primary metrics used throughout the paper: total operational cost $C_{\mathrm{tot}}$, safety violations $V$, detection delay $T_{\mathrm{det}}$, recovery time $T_{\mathrm{rec}}$, and minimum worst-group accuracy over drift windows. Figure~\ref{fig:pareto_frontier} visualizes the safety--cost trade-off by plotting $V$ against $C_{\mathrm{tot}}$ for each method, showing that the certified controller occupies the low-violation region at moderate cost. Figure~\ref{fig:recovery_curves} illustrates the temporal recovery dynamics around a drift onset, including intervention markers that indicate when major actions are executed (see Appendix \ref{app:extra_experiments} for more results).

\begin{table}[!t]
\caption{\textbf{Main streaming results under operational constraints.} We report total cost $C_{\mathrm{tot}}$, safety violations $V$, detection delay $T_{\mathrm{det}}$, recovery time $T_{\mathrm{rec}}$, and minimum worst-group accuracy $\min \mathrm{Acc}^{\mathrm{wg}}$ across Camelyon17, DomainNet, and SyntheticDrift-CIFAR using $\tau=0.20$, $\delta=0.05$, and label delay $d=50$.}
\label{tab:main_results}
\centering
\small
\resizebox{0.6\linewidth}{!}{%
\begin{tabular}{p{0.27\linewidth} p{0.12\linewidth} p{0.10\linewidth} p{0.10\linewidth} p{0.10\linewidth} p{0.13\linewidth}}
\hline
Method & $C_{\mathrm{tot}}$ & $V$ & $T_{\mathrm{det}}$ & $T_{\mathrm{rec}}$ & $\min \mathrm{Acc}^{\mathrm{wg}}$ \\
\hline
\multicolumn{6}{l}{\textbf{Camelyon17 stream}}\\
Alarm-only & 3.2 & 46 & 22 & 210 & 0.71 \\
Adapt-always (TTA) & 58.0 & 3 & 18 & 74 & 0.79 \\
Retrain-on-schedule & 41.5 & 9 & 35 & 96 & 0.77 \\
Selective prediction only & 9.8 & 14 & 24 & 160 & 0.74 \\
Controller (no certificate) & 18.7 & 7 & 19 & 98 & 0.78 \\
Controller + certificate (ours) & 24.6 & 0 & 20 & 62 & 0.81 \\
\hline
\multicolumn{6}{l}{\textbf{DomainNet stream}}\\
Alarm-only & 4.1 & 63 & 28 & 260 & 0.49 \\
Adapt-always (TTA) & 64.0 & 6 & 21 & 110 & 0.58 \\
Retrain-on-schedule & 45.0 & 14 & 40 & 132 & 0.56 \\
Selective prediction only & 12.4 & 27 & 30 & 205 & 0.52 \\
Controller (no certificate) & 22.5 & 11 & 23 & 141 & 0.57 \\
Controller + certificate (ours) & 29.8 & 1 & 24 & 88 & 0.61 \\
\hline
\multicolumn{6}{l}{\textbf{SyntheticDrift-CIFAR stream}}\\
Alarm-only & 2.6 & 52 & 16 & 185 & 0.68 \\
Adapt-always (TTA) & 52.0 & 4 & 13 & 66 & 0.79 \\
Retrain-on-schedule & 36.0 & 10 & 25 & 92 & 0.77 \\
Selective prediction only & 8.6 & 19 & 18 & 148 & 0.72 \\
Controller (no certificate) & 16.8 & 8 & 14 & 84 & 0.78 \\
Controller + certificate (ours) & 21.2 & 0 & 15 & 58 & 0.82 \\
\hline
\end{tabular}}
\end{table}


\begin{figure}[htbp]
\centering
\begin{subfigure}[t]{0.42\linewidth}
\centering
\includegraphics[width=\linewidth]{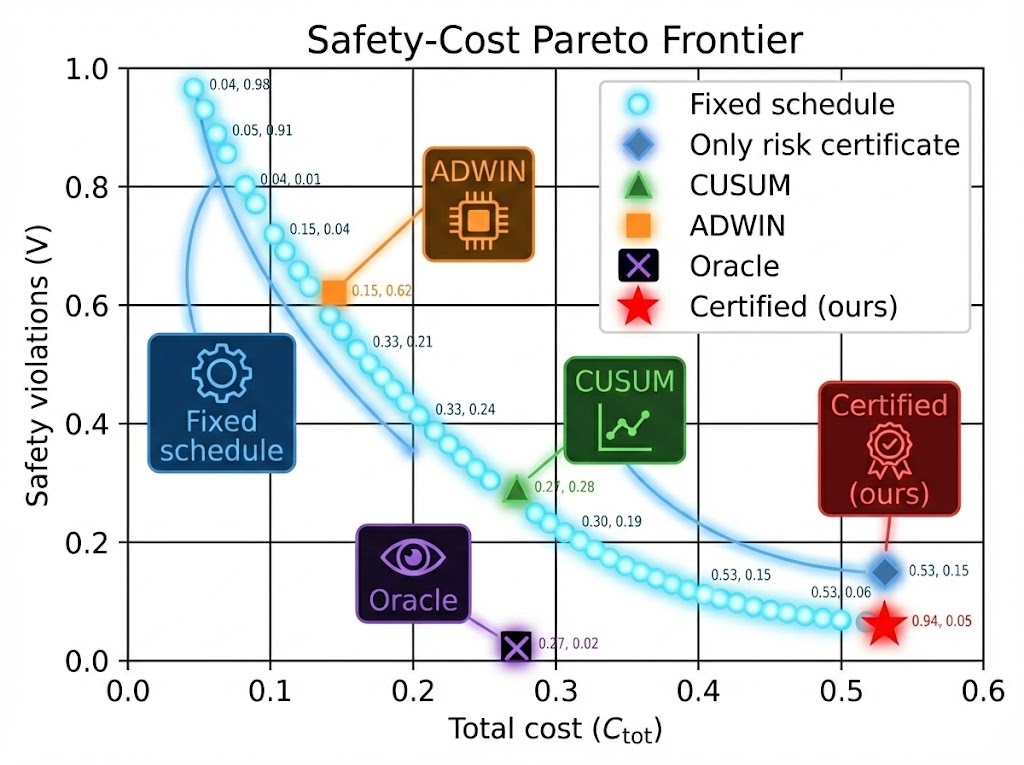}
\vspace{-6pt}
\caption{\textbf{Safety--cost Pareto frontier.} Methods on the Camelyon17 stream plotted by total cost $C_{\mathrm{tot}}$ (x-axis) and safety violations $V$ (y-axis); the certified controller attains low violations at moderate cost.}
\label{fig:pareto_frontier}
\end{subfigure}\hfill
\begin{subfigure}[t]{0.48\linewidth}
\centering
\includegraphics[width=\linewidth]{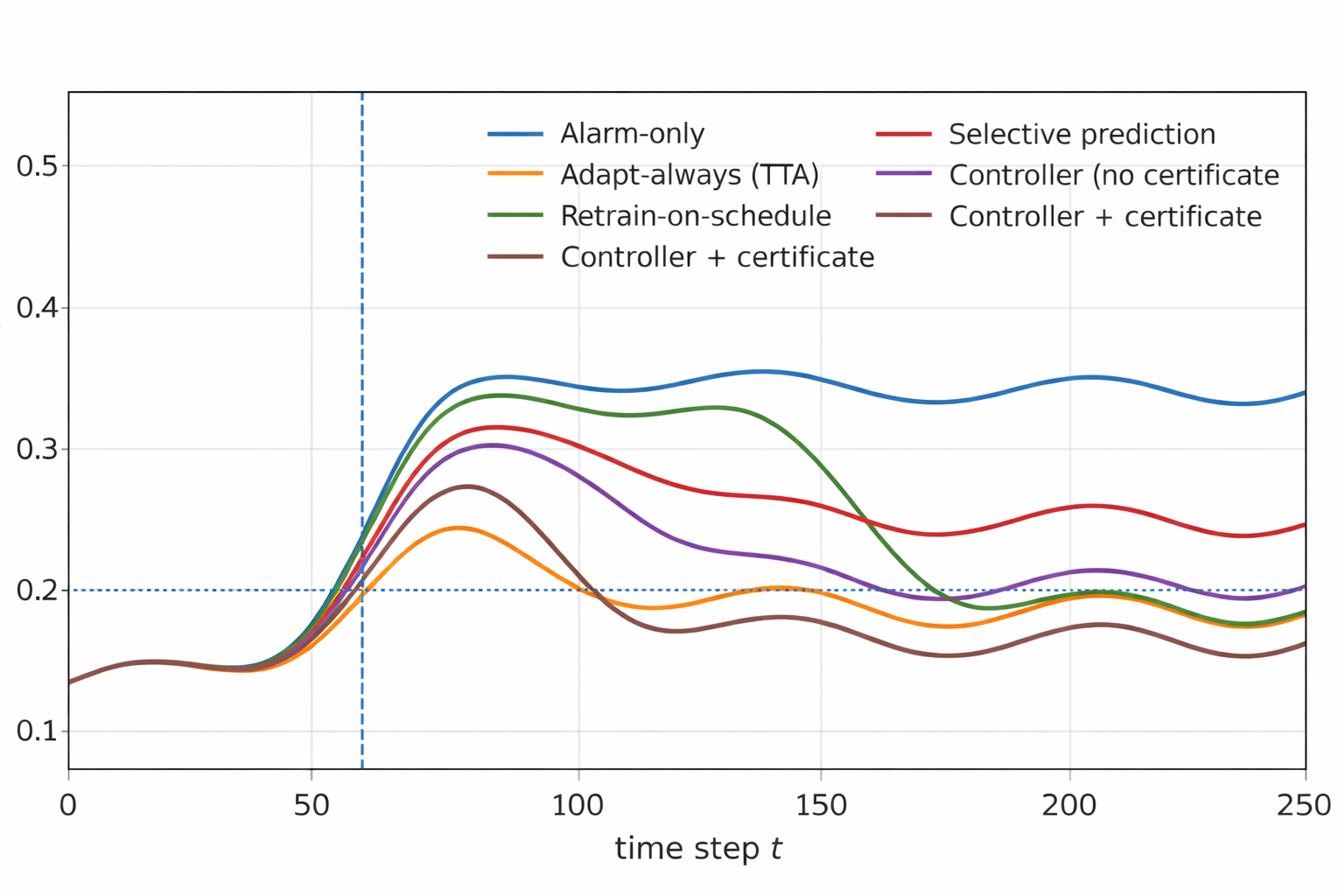}
\caption{\textbf{Recovery curves around a drift event.} Deployed risk $R_t^{\star}$ over time with drift onset (vertical dashed line) and safety threshold $\tau$ (horizontal dotted line); tick marks indicate interventions taken by the certified controller.}
\label{fig:recovery_curves}
\end{subfigure}

\vspace{-6pt}
\caption{\textbf{Certified controller performance under cost--safety trade-offs and drift.} (Left) Safety--cost Pareto frontier on Camelyon17. (Right) Recovery dynamics around a drift event relative to threshold $\tau$, with intervention times marked.}
\label{fig:combined_fig2}
\vspace{-8pt}
\end{figure}

\section{Conclusion}
\label{sec:conclusion}

We reframed drift monitoring as decision-making under operational constraints. Our drift-to-action controller combines drift note: a belief over drift types from unlabeled monitors, an active risk certificate that provides an anytime-valid upper bound on current risk from a small set of randomly sampled delayed labels, and a cost-aware policy that selects interventions under budgets and cooldowns. In realistic streaming evaluations with label delay, this reduces safety violations and speeds recovery at moderate cost compared to alarm-only monitoring, always-on adaptation, and schedule-based retraining.

\bibliography{iclr2026_conference}
\bibliographystyle{iclr2026_conference}
\newpage
\appendix
\section*{Appendix}
\label{appen}
\section{Problem Setup and Operating Constraints}
\label{sec:problem_setup}

We consider a deployed learning system that processes a potentially non-stationary stream over a finite horizon $t\in\{1,\dots,T\}$. At each time step $t$, the system receives an input $x_t\in\mathcal{X}$, produces an output (prediction or abstention), and may later receive the corresponding label $y_t\in\mathcal{Y}$. Distribution drift is captured by allowing the deployment distribution $\mathcal{D}_t$ governing $(x_t,y_t)$ to change with time.

\subsection{Streaming data with delayed supervision}
\label{sec:streaming_delay}

We model delayed supervision with an observation time variable $o_t \in \{t,t+1,\dots\}\cup\{\infty\}$ indicating when the label $y_t$ becomes available. The event $o_t=\infty$ represents that $y_t$ is never observed. The set of labeled examples available by time $t$ is
\begin{equation}
\mathcal{S}_t \;=\; \{(x_i,y_i)\;:\; 1 \le i \le t \ \text{and}\ o_i \le t\}.
\label{eq33}
\end{equation}
In this definition, $i$ indexes an earlier time step at which input $x_i$ was observed, the condition $o_i \le t$ indicates that its label $y_i$ has arrived by time $t$, and $\mathcal{S}_t$ is the only supervised data accessible for calibration, adaptation, and retraining decisions at time $t$. The evolving data-generating process is represented by a time-indexed distribution $\mathcal{D}_t$ for pairs $(x_t,y_t)$, allowing temporal, covariate, label, concept, or subgroup shifts.

\subsection{Predictive model and outputs}
\label{sec:base_model}

The deployed predictor is parameterized by $\theta$ and outputs a categorical distribution:
\begin{equation}
p_{\theta}(y \mid x) \;=\; \mathrm{softmax}\!\big(h_{\theta}(x)\big).
\end{equation}
Here $h_{\theta}:\mathcal{X}\rightarrow\mathbb{R}^{|\mathcal{Y}|}$ maps an input $x$ to logits, $\mathrm{softmax}(\cdot)$ converts logits into class probabilities, $p_{\theta}(y \mid x)$ is the predicted probability assigned to label $y\in\mathcal{Y}$, and $\theta$ denotes all model parameters. The system's point prediction is
\begin{equation}
\hat{y}_{\theta}(x) \;=\; \arg\max_{y\in\mathcal{Y}} p_{\theta}(y\mid x),
\end{equation}
where $\hat{y}_{\theta}(x)$ is the most probable class under the predictive distribution.

\subsection{Drift evidence from unlabeled monitoring}
\label{sec:drift_evidence}

To sense drift without requiring immediate labels, the system maintains two unlabeled buffers: a reference buffer $\mathcal{B}_{\mathrm{ref}}$ representing nominal conditions and a recent buffer $\mathcal{B}_t$ of size $n$ containing the most recent inputs,
\begin{equation}
\mathcal{B}_t \;=\; \{x_{t-n+1},\dots,x_t\}.
\end{equation}
A suite of monitoring functions computes statistics comparing these buffers, producing an evidence vector
\begin{equation}
z_t \;=\; \big(M_1(\mathcal{B}_t,\mathcal{B}_{\mathrm{ref}}),\dots,M_m(\mathcal{B}_t,\mathcal{B}_{\mathrm{ref}})\big) \ \in\ \mathbb{R}^{m}.
\end{equation}
In this equation, $M_j(\cdot,\cdot)$ denotes the $j$-th monitor, such as a representation discrepancy measure, a two-sample classifier score, or an uncertainty shift statistic. The vector $z_t$ aggregates heterogeneous drift signals into a fixed-dimensional summary that can be computed online and passed to the belief model and controller.

\subsection{Interventions and action space}
\label{sec:actions}

At each time step $t$, the system selects an intervention $a_t$ from a discrete action set $\mathcal{A}$. Actions represent operational responses available in real deployments, including inexpensive adjustments (e.g., recalibration), medium-cost responses (e.g., test-time adaptation), and heavy-weight interventions (e.g., retraining or rollback). Table~\ref{tab:action_space_costs} defines each action precisely and assigns a normalized operational cost used throughout the paper.

\begin{table}[htbp]
\centering
\caption{\textbf{Action space and normalized operational costs used by the controller.} }
\label{tab:action_space_costs}
\small
\begin{tabular}{p{0.10\linewidth} p{0.20\linewidth} p{0.44\linewidth} p{0.18\linewidth}}
\hline
Action & Name & Operational effect & Cost (units) \\
\hline
$A_0$ & No-op & Continue inference with current parameters $\theta$ and current output policy. & $0.0$ \\
$A_1$ & Recalibrate & Update a lightweight calibrator using $\mathcal{S}_t$; logits unchanged. & $0.2$ \\
$A_2$ & Test-time adapt & Apply a small number of adaptation steps on $\mathcal{B}_t$; update $\theta \leftarrow \theta'$. & $1.0$ \\
$A_3$ & Query labels & Request $k_t$ labels from recent inputs; add to $\mathcal{S}_{t'}$ when they arrive. & $0.05\,k_t$ \\
$A_4$ & Retrain & Trigger retraining using $\mathcal{S}_t$ plus logged data; deploy updated $\theta$. & $12.0$ \\
$A_5$ & Rollback & Revert to a previously validated checkpoint $\theta_{\mathrm{safe}}$. & $1.5$ \\
$A_6$ & Abstain/handoff & Abstain on selected inputs and route to fallback (e.g., human review), reducing coverage. & $0.3$ \\
\hline
\end{tabular}
\end{table}

\subsection{Cost model}
\label{sec:costs}

We represent deployment burden with a nonnegative cost function $c:\mathcal{A}\rightarrow\mathbb{R}_{\ge 0}$. To reflect distinct operational constraints, we decompose cost into labeling, compute, and latency components:
\begin{equation}
c(a_t) \;=\; c_{\mathrm{label}}(a_t)\;+\;c_{\mathrm{comp}}(a_t)\;+\;c_{\mathrm{lat}}(a_t).
\end{equation}
In this equation, $c_{\mathrm{label}}(a_t)$ quantifies human annotation or review effort (relevant for $A_3$ and potentially $A_6$), $c_{\mathrm{comp}}(a_t)$ measures computational overhead (dominant for $A_2$ and $A_4$), and $c_{\mathrm{lat}}(a_t)$ captures latency or throughput degradation induced by the intervention. The cumulative operational cost over the deployment horizon is
\begin{equation}
C_{\mathrm{tot}} \;=\; \sum_{t=1}^{T} c(a_t),
\end{equation}
where $c(a_t)$ is the per-step cost and $C_{\mathrm{tot}}$ summarizes the overall resource and process burden.

\subsection{Safety objective under drift}
\label{sec:safety_objective}

We evaluate predictive quality using a bounded loss function $\ell(\hat{y},y)\in[0,1]$, such as the zero-one loss for classification. The population risk at time $t$ is
\begin{equation}
R_t \;=\; \mathbb{E}_{(x,y)\sim\mathcal{D}_t}\Big[\ell\big(\hat{y}_{\theta}(x),y\big)\Big].
\end{equation}
In this definition, $\mathcal{D}_t$ is the time-dependent deployment distribution, $\hat{y}_{\theta}(x)$ is the model's predicted label under parameters $\theta$, and $\ell(\cdot,\cdot)$ is the task loss. The primary safety requirement is to keep risk below a target threshold $\tau$ whenever the system produces an automated prediction:
\begin{equation}
R_t \;\le\; \tau.
\end{equation}
When the system can abstain and hand off uncertain cases, it uses an acceptance function $s_t:\mathcal{X}\rightarrow\{0,1\}$, where $s_t(x)=1$ indicates that the system predicts and $s_t(x)=0$ indicates abstention. The selective risk and coverage are
\begin{equation}
R_t^{\mathrm{sel}} \;=\; \mathbb{E}_{(x,y)\sim\mathcal{D}_t}\Big[\ell\big(\hat{y}_{\theta}(x),y\big)\ \big|\ s_t(x)=1\Big],
\end{equation}
\begin{equation}
\kappa_t \;=\; \mathbb{P}_{(x,y)\sim\mathcal{D}_t}\big(s_t(x)=1\big).
\end{equation}
In these equations, $R_t^{\mathrm{sel}}$ measures error conditional on the system choosing to predict, and $\kappa_t$ is the fraction of inputs receiving automated predictions. The safety requirement with abstention is
\begin{equation}
R_t^{\mathrm{sel}} \;\le\; \tau
\qquad \text{and} \qquad
\kappa_t \;\ge\; \kappa_{\min},
\end{equation}
where $\kappa_{\min}$ enforces a minimum coverage level, preventing degenerate solutions that satisfy safety by abstaining on nearly all inputs.

\subsection{Evaluation metrics}
\label{sec:metrics}

We evaluate monitoring-and-control methods using metrics that reflect both reliability and operational constraints. In controlled streaming evaluations, we denote by $t_0$ the onset time of a drift event. Let $\mathrm{Alarm}_t\in\{0,1\}$ indicate whether the monitoring stack triggers an alarm at time $t$. The detection delay is
\begin{equation}
T_{\mathrm{det}} \;=\; \min\{t \ge t_0 : \mathrm{Alarm}_t=1\} \;-\; t_0.
\end{equation}
Here $T_{\mathrm{det}}$ is the number of time steps between drift onset and the first alarm, measuring how quickly the monitoring layer detects change.

To measure how quickly the system returns to safe operation after drift, we define a deployed risk metric $R_t^{\star}$ that equals $R_t$ when the system always predicts and equals $R_t^{\mathrm{sel}}$ when abstention is enabled. The recovery time is
\begin{equation}
T_{\mathrm{rec}} \;=\; \min\{t \ge t_0 : R_t^{\star} \le \tau\} \;-\; t_0.
\end{equation}
In this equation, $T_{\mathrm{rec}}$ counts the time steps from drift onset until the deployed risk metric falls below the safety threshold $\tau$.

We quantify unsafe operation through the number of safety violations over the horizon:
\begin{equation}
V \;=\; \sum_{t=1}^{T} \mathbb{I}\!\left(R_t^{\star}>\tau\right).
\end{equation}
Here $\mathbb{I}(\cdot)$ is an indicator function, and $V$ counts the total number of time steps where deployed risk exceeds the target threshold.

To assess subgroup robustness under drift, we assume a set of groups $\mathcal{G}$, such as demographic groups, domains, or discovered clusters. The worst-group accuracy at time $t$ is
\begin{equation}
\mathrm{Acc}_t^{\mathrm{wg}} \;=\; \min_{g\in\mathcal{G}} \ \mathbb{P}_{(x,y)\sim\mathcal{D}_t}\big(\hat{y}_{\theta}(x)=y \ \big|\ g(x)=g\big).
\end{equation}
In this definition, $g(x)$ denotes the group membership of input $x$, the conditional probability measures accuracy within group $g$, and the minimum over $\mathcal{G}$ isolates the most adversely affected group at time $t$. Alongside $T_{\mathrm{det}}$, $T_{\mathrm{rec}}$, and $V$, we report the cumulative operational cost $C_{\mathrm{tot}}$ from Eq.~(\ref{eq8}), yielding a joint view of reliability, responsiveness, and operational burden.

\section{Algorithm}
\label{sec:algorithm}

Algorithm~\ref{alg:controller} summarizes the complete online loop. The loop updates monitoring signals, drift-type belief, and the active risk certificate; it then selects an action under feasibility constraints, applies the action, and logs outcomes for auditing and later analysis.

\begin{algorithm}[!t]
\caption{Drift-to-Action Control with Active Risk Certificate}
\label{alg:controller}
\KwIn{Reference buffer $\mathcal{B}_{\mathrm{ref}}$, monitor window size $n$, certificate window size $N$, risk target $\tau$, failure level $\delta$, labeling budget $B_{\mathrm{lab}}$, retrain cooldown $\Delta_{\mathrm{rt}}$, rollback cooldown $\Delta_{\mathrm{rb}}$, cost tradeoff $\lambda$, gain table $G(d,a)$}
\KwOut{Action sequence $\{a_t\}_{t=1}^{T}$ and audit log $\mathcal{L}$}

Initialize model parameters $\theta$\;
Initialize calibrator parameters $\eta$\;
Initialize drift-type belief $b_0(d)$ for all $d\in\mathcal{D}$ with $\sum_{d\in\mathcal{D}} b_0(d)=1$\;
Initialize labeled set $\mathcal{S}_0 \leftarrow \emptyset$ and audit log $\mathcal{L}\leftarrow \emptyset$\;
Initialize last intervention times $t_{\mathrm{last}}^{\mathrm{rt}}\leftarrow -\infty$ and $t_{\mathrm{last}}^{\mathrm{rb}}\leftarrow -\infty$\;
Initialize remaining labeling budget $B \leftarrow B_{\mathrm{lab}}$\;

\For{$t\leftarrow 1$ \KwTo $T$}{
Receive input $x_t$ and update recent buffer $\mathcal{B}_t=\{x_{t-n+1},\dots,x_t\}$\;
Compute evidence vector $z_t=\big(M_1(\mathcal{B}_t,\mathcal{B}_{\mathrm{ref}}),\dots,M_m(\mathcal{B}_t,\mathcal{B}_{\mathrm{ref}})\big)$\;
Update belief $b_t(d)=\mathbb{P}(D_t=d\mid z_{1:t})$ for all $d\in\mathcal{D}$\;

Construct certifiable window $W_t=\{t-d-N+1,\dots,t-d\}$\;
Choose audit size $n_t \leftarrow \mathrm{QueryPolicy}(z_t,b_t,U_{t-1}(\delta),B)$\;
Uniformly sample audit indices $Q_t \subseteq W_t$ with $|Q_t|=n_t$\;

$J_t \leftarrow Q_t \setminus \{i:\ (x_i,y_i)\in\mathcal{S}_{t-1}\}$\;
Reveal/request labels $\{y_i : i\in J_t\}$ and update budget $B\leftarrow B-|J_t|$\;
Update $\mathcal{S}_t \leftarrow \mathcal{S}_{t-1}\cup\{(x_i,y_i): i\in J_t\}$\;

Compute queried losses $\ell_i=\ell(\hat{y}_{\theta}(x_i),y_i)$ for all $i\in Q_t$\;
Compute empirical audited risk $\widehat{R}_t=\frac{1}{|Q_t|}\sum_{i\in Q_t}\ell_i$\;
Set $\delta_t=\frac{6\delta}{\pi^2 t^2}$\;
Compute certificate $U_t(\delta)=\widehat{R}_t+\mathrm{rad}(|Q_t|,\delta_t)$\;

\eIf{$U_t(\delta)>\tau$}{
Set action $a_t \leftarrow A_6$ \tcp*[r]{abstain/handoff}
\If{$t-t_{\mathrm{last}}^{\mathrm{rb}}\ge \Delta_{\mathrm{rb}}$}{
Schedule rollback and update $t_{\mathrm{last}}^{\mathrm{rb}}\leftarrow t$\;
}
\ElseIf{$t-t_{\mathrm{last}}^{\mathrm{rt}}\ge \Delta_{\mathrm{rt}}$}{
Schedule retraining and update $t_{\mathrm{last}}^{\mathrm{rt}}\leftarrow t$\;
}
}{
Form feasible action set $\mathcal{A}_t^{\mathrm{feas}}\subseteq\mathcal{A}$ using cooldowns and remaining budget $B$\;
\ForEach{$a\in \mathcal{A}_t^{\mathrm{feas}}$}{
Compute expected benefit $\Delta_t(a)=\sum_{d\in\mathcal{D}} b_t(d)\,G(d,a)$\;
Compute utility $\mathcal{U}_t(a)=\Delta_t(a)-\lambda\,c(a)$\;
}
Choose action $a_t \leftarrow \arg\max_{a\in\mathcal{A}_t^{\mathrm{feas}}}\mathcal{U}_t(a)$\;
}

Apply action $a_t$:\;
\If{$a_t=A_1$}{Update calibrator $\eta$ using $\mathcal{S}_t$\;}
\If{$a_t=A_2$}{Perform test-time adaptation on recent inputs; update $\theta\leftarrow\theta'$\;}
\If{$a_t=A_4$}{Retrain model using $\mathcal{S}_t$ and logged data; deploy updated $\theta$\;}
\If{$a_t=A_5$}{Rollback to checkpoint $\theta_{\mathrm{safe}}$ and associated calibrator\;}
\If{$a_t=A_6$}{Abstain/handoff using policy $s_t(x)$\;}

Log $(t,z_t,b_t,U_t(\delta),a_t,c(a_t),k_t)$ into $\mathcal{L}$\;
}
\end{algorithm}

\section{Experiments}
\label{sec:experiments}
We evaluate drift monitoring and response in a streaming setting where data arrive sequentially and the deployed system must decide when to recalibrate, adapt, request labels, retrain, roll back, or fall back (see Appendix \ref{app:impl} for more details). The stream is constructed by concatenating either time-ordered chunks (when timestamps exist) or environment-ordered chunks (when datasets are organized by domains). Let $\{\mathcal{E}_1,\dots,\mathcal{E}_K\}$ denote environments (e.g., hospitals in Camelyon17 or domains in DomainNet). The stream assigns each interval to one environment:
\begin{equation}
(x_t,y_t)\sim \mathcal{D}_t,
\qquad
\mathcal{D}_t=\mathcal{D}^{(k)}\ \ \text{for}\ \ t\in\{T_{k-1}+1,\dots,T_k\}.
\end{equation}
In this equation, $\mathcal{D}^{(k)}$ is the environment-specific distribution for $\mathcal{E}_k$, $T_k$ is the final time index of environment $k$, and $T_0=0$. This protocol produces realistic distribution transitions while preserving within-environment temporal coherence.

To test robustness beyond discrete domain switches, we also use mixture-controlled drift schedules that interpolate between a nominal distribution $\mathcal{D}^{\mathrm{ref}}$ and a shifted distribution $\mathcal{D}^{\mathrm{shift}}$:
\begin{equation}
\mathcal{D}_t \;=\; (1-\alpha_t)\,\mathcal{D}^{\mathrm{ref}} \;+\; \alpha_t\,\mathcal{D}^{\mathrm{shift}}.
\end{equation}
Here $\alpha_t\in[0,1]$ is a time-varying drift intensity, with $\alpha_t=0$ representing nominal conditions and $\alpha_t=1$ representing fully shifted conditions. We instantiate three drift patterns via $\alpha_t$. Sudden drift uses $\alpha_t=\mathbb{I}(t\ge t_0)$ where $t_0$ is drift onset. Gradual drift uses a sigmoid ramp
\begin{equation}
\alpha_t \;=\; \frac{1}{1+\exp\!\big(-\rho\,(t-t_0)\big)},
\end{equation}
where $\rho>0$ controls the ramp speed. Recurring drift uses a periodic schedule
\begin{equation}
\alpha_t \;=\; \frac{1}{2}\Big(1+\sin(2\pi t/P)\Big),
\end{equation}
where $P$ is the period length in time steps, producing repeated returns to nominal conditions and repeated drift episodes.

Operational constraints are integrated directly into the evaluation. Each action $a_t\in\mathcal{A}$ incurs a normalized cost $c(a_t)$ and heavy interventions are rate-limited by cooldowns. The total deployment burden is measured by
\begin{equation}
C_{\mathrm{tot}} \;=\; \sum_{t=1}^{T} c(a_t),
\end{equation}
where $T$ is the stream length. We also simulate realistic annotation pipelines through label delay. If the controller requests the label for index $i$, the label becomes available after $d$ steps:
\begin{equation}
o_i \;=\; i + d,
\label{eq32}
\end{equation}
where $o_i$ is the observation time at which $y_i$ arrives. This delay affects both calibration monitoring and risk certification, because losses can only be computed once labels are observed. Table~\ref{tab:protocol_costs} specifies the concrete action costs, prerequisites, and cooldowns used by all methods in the protocol, and it is referenced by the controller when enforcing feasibility constraints.


\begin{wraptable}[16]{r}{0.4\linewidth}
\captionsetup{
  font=footnotesize,
  skip=2pt,
  width=\linewidth,          
  format=plain,
  justification=raggedright,
  singlelinecheck=false,
  margin=0pt
}
\caption{\textbf{Streaming protocol action costs and prerequisites.} Normalized costs and cooldowns.}
\label{tab:protocol_costs}

\footnotesize
\setlength{\tabcolsep}{1.5pt}
\renewcommand{\arraystretch}{0.75}
\begin{tabular}{@{}l l p{1.75cm} c c@{}}
\toprule
Act. & Name & Prereq./Effect & Cooldown & Cost \\
\midrule
$A_0$ & No-op    & Keep $\theta$. & -- & $0.0$ \\
$A_1$ & Recalib. & Update from $\mathcal{S}_t$. & -- & $0.2$ \\
$A_2$ & TTA      & Adapt on $\mathcal{B}_t$. & -- & $1.0$ \\
$A_3$ & Query    & Request $k_t$; delay $d$. & -- & $0.05k_t$ \\
$A_4$ & Retrain  & Retrain + deploy. & $\Delta_{\mathrm{rt}}$ & $12.0$ \\
$A_5$ & Rollback & Restore $\theta_{\mathrm{safe}}$. & $\Delta_{\mathrm{rb}}$ & $1.5$ \\
$A_6$ & Abstain  & Handoff; less coverage. & -- & $0.3$ \\
\bottomrule
\end{tabular}
\end{wraptable}

\section{Full theoretical guarantee}
\label{app:confidence_sequence}

This appendix provides a full, self-contained guarantee for the active risk certificate used in Section~\ref{sec:active_risk_certificate}. We first state the probabilistic model under which the certificate is valid, then derive an anytime-valid upper confidence sequence for the mean of bounded losses under adaptive (optional-stopping) label querying, and finally extend the result to selective risk when abstention is enabled.

\subsection{Windowed risk, random sampling, and filtration}
\label{app:setup}

Fix a time $t$ and a window length $N\in\mathbb{N}$. The current deployment regime is represented by the index set
\begin{equation}
W_t \;=\; \{t-N+1,\dots,t\}.
\end{equation}
In this equation, $t$ is the current time index and $W_t$ contains the most recent $N$ indices.

Each index $i\in W_t$ corresponds to an example $(x_i,y_i)$ and an associated bounded loss
\begin{equation}
L_{t,i} \;=\; \ell\!\big(\hat{y}_{\theta}(x_i),y_i\big),
\qquad
0 \le L_{t,i} \le 1.
\end{equation}
In this equation, $\ell(\cdot,\cdot)$ is the task loss, $\hat{y}_{\theta}(x_i)$ is the model prediction at index $i$, and the bound $[0,1]$ holds for standard classification losses such as $\mathbb{I}(\hat{y}_{\theta}(x_i)\neq y_i)$.

The windowed (finite-population) risk is the mean loss over the current window:
\begin{equation}
R_t \;=\; \frac{1}{N}\sum_{i\in W_t} L_{t,i}.
\end{equation}
In this equation, $R_t$ is the quantity we want to upper bound online.

The certificate queries $n_t$ labels by sampling indices uniformly at random from $W_t$. Let
\begin{equation}
I_{t,1},I_{t,2},\dots,I_{t,n_t} \;\sim\; \mathrm{Unif}(W_t),
\end{equation}
where $\mathrm{Unif}(W_t)$ denotes the uniform distribution over the set $W_t$. The sampling may be done with replacement; sampling without replacement is discussed in Section~\ref{app:serfling}. The observed queried losses are
\begin{equation}
X_{t,j} \;=\; L_{t,I_{t,j}},
\qquad j\in\{1,\dots,n_t\}.
\end{equation}
In this equation, $X_{t,j}$ is the loss of the $j$-th randomly sampled index within the window at time $t$.

The certificate uses the empirical mean of the queried losses,
\begin{equation}
\widehat{R}_{t,n} \;=\; \frac{1}{n}\sum_{j=1}^{n} X_{t,j},
\end{equation}
where $n$ is the number of queried labels actually used in the estimate.

To capture adaptive label querying, we define a filtration that represents all information available after $n$ queried labels in window $t$:
\begin{equation}
\mathcal{F}_{t,n} \;=\; \sigma\!\left(z_{1:t},\, I_{t,1},X_{t,1},\,\dots,\, I_{t,n},X_{t,n}\right).
\end{equation}
In this equation, $z_{1:t}$ is the monitor history and $\sigma(\cdot)$ denotes the $\sigma$-algebra generated by the listed random variables. The query count $n_t$ is allowed to be an $\mathcal{F}_{t,n}$-adapted (data-dependent) stopping rule, meaning $n_t$ can depend on past observed evidence and queried labels.

The key validity condition used throughout this appendix is the following.

\textbf{Condition A (Random auditing within the current window).}
At each time $t$, the queried indices $I_{t,1},\dots,I_{t,n_t}$ are sampled uniformly from $W_t$, and the sampling mechanism is independent of the unobserved labels given the already observed history.

This condition allows the certificate to treat the queried losses as an unbiased sample from the window mean risk $R_t$, even when the decision to query more labels is adaptive.

\subsection{A time-uniform confidence sequence for bounded means}
\label{app:cs}

We derive an anytime-valid upper confidence sequence for a bounded mean under adaptive stopping. The statement is formulated for a generic bounded sequence and then specialized to the windowed risk setting.

Let $X_1,X_2,\dots$ be random variables satisfying
\begin{equation}
0 \le X_j \le 1,
\qquad
\mathbb{E}[X_j]=\mu,
\qquad
\text{and $X_j$ are i.i.d.}.
\end{equation}
In this equation, $\mu$ is the unknown mean to be bounded. Define the sample mean
\begin{equation}
\widehat{\mu}_n \;=\; \frac{1}{n}\sum_{j=1}^{n} X_j.
\end{equation}
In this equation, $n$ is the number of samples and $\widehat{\mu}_n$ is their empirical average.

We use Hoeffding's inequality at each fixed $n$:
\begin{equation}
\mathbb{P}\!\left(\mu \ge \widehat{\mu}_n + \epsilon\right)
\;\le\;
\exp\!\left(-2n\epsilon^2\right).
\end{equation}
In this inequality, $\epsilon>0$ is a deviation level and the bound relies only on the fact that $X_j\in[0,1]$.

To make the bound \emph{anytime-valid} over all $n$, we allocate failure probability across time using a summable schedule. Let
\begin{equation}
\delta_n \;=\; \frac{6\delta}{\pi^2 n^2},
\qquad
\delta\in(0,1).
\end{equation}
In this equation, $\delta$ is the target overall failure probability and $\delta_n$ is the portion allocated to sample size $n$. The constant is chosen so that the sum is bounded:
\begin{equation}
\sum_{n=1}^{\infty}\delta_n
\;=\;
\frac{6\delta}{\pi^2}\sum_{n=1}^{\infty}\frac{1}{n^2}
\;=\;
\delta.
\end{equation}
In this equation, $\sum_{n\ge 1}1/n^2=\pi^2/6$ is the Basel identity, ensuring the total allocated failure probability is exactly $\delta$.

Define the confidence radius
\begin{equation}
\mathrm{rad}(n,\delta)
\;=\;
\sqrt{\frac{1}{2n}\log\!\left(\frac{1}{\delta_n}\right)}
\;=\;
\sqrt{\frac{1}{2n}\log\!\left(\frac{\pi^2 n^2}{6\delta}\right)}.
\end{equation}
In this equation, $n$ is the sample size and $\mathrm{rad}(n,\delta)$ shrinks as $n$ increases.

\begin{lemma}\label{lemmaA.1}
\textbf{(Anytime-valid upper confidence sequence via stitching).} 
For any $\delta\in(0,1)$ and the radius $\mathrm{rad}(n,\delta)$ defined above, it holds that
\begin{equation}
\mathbb{P}\!\left(\forall n\ge 1:\ \mu \le \widehat{\mu}_n + \mathrm{rad}(n,\delta)\right) \;\ge\; 1-\delta.
\end{equation}
In this probability statement, the event inside holds \emph{simultaneously} for all sample sizes $n$, which implies validity under optional stopping.
\end{lemma}

\textit{Proof.}
For any fixed $n$, set $\epsilon=\mathrm{rad}(n,\delta)$ in Hoeffding's inequality to obtain
\begin{equation}
\mathbb{P}\!\left(\mu \ge \widehat{\mu}_n + \mathrm{rad}(n,\delta)\right)
\;\le\;
\exp\!\left(-2n\,\mathrm{rad}(n,\delta)^2\right)
\;=\;
\exp\!\left(-\log\!\left(\frac{1}{\delta_n}\right)\right)
\;=\;
\delta_n.
\end{equation}
In this equation chain, the first step applies Hoeffding, the second step substitutes the definition of $\mathrm{rad}(n,\delta)$, and the last step simplifies the exponential.

Define the bad event at time $n$ by
\begin{equation}
E_n \;=\; \left\{\mu \ge \widehat{\mu}_n + \mathrm{rad}(n,\delta)\right\}.
\end{equation}
Then, by the union bound,
\begin{equation}
\mathbb{P}\!\left(\exists n\ge 1:\ E_n\right)
\;\le\;
\sum_{n=1}^{\infty}\mathbb{P}(E_n)
\;\le\;
\sum_{n=1}^{\infty}\delta_n
\;=\;
\delta.
\end{equation}
In this bound, the first inequality uses the union bound, the second uses $\mathbb{P}(E_n)\le\delta_n$, and the last uses the summability property shown earlier. Taking complements yields the claim. \hfill $\square$

A crucial consequence of \lemref{lemmaA.1} is optional-stopping validity. Let $\widehat{n}$ be any (possibly random) stopping time that depends on the observed data, such as an adaptive query rule that stops once the bound is sufficiently tight. Since the event in \lemref{lemmaA.1} holds for all $n$ simultaneously, it also holds at $n=\widehat{n}$:
\begin{equation}
\mathbb{P}\!\left(\mu \le \widehat{\mu}_{\widehat{n}} + \mathrm{rad}(\widehat{n},\delta)\right)\;\ge\; 1-\delta.
\end{equation}
In this equation, $\widehat{n}$ is data-dependent, yet the bound remains valid because the confidence sequence is time-uniform.

\subsection{Applying the confidence sequence to windowed risk certificates}
\label{app:apply}

We now map the generic result to the windowed risk certificate at time $t$. Under Condition~A, the queried losses $X_{t,1},\dots,X_{t,n}$ satisfy
\begin{equation}
0 \le X_{t,j} \le 1,
\qquad
\mathbb{E}[X_{t,j}\mid W_t] \;=\; R_t,
\end{equation}
where $R_t$ is the window mean risk. The empirical risk based on $n$ queried labels is
\begin{equation}
\widehat{R}_{t,n} \;=\; \frac{1}{n}\sum_{j=1}^{n}X_{t,j}.
\end{equation}
In this equation, $\widehat{R}_{t,n}$ is the empirical mean of queried losses within the window at time $t$.

Define the per-time certificate using the confidence radius:
\begin{equation}
U_{t,n}(\delta_t)
\;=\;
\widehat{R}_{t,n} + \mathrm{rad}(n,\delta_t),
\end{equation}
where $\delta_t\in(0,1)$ is the failure probability budget assigned to time $t$.

To obtain a guarantee that holds simultaneously over all times $t\in\{1,\dots,T\}$, we allocate the overall failure probability $\delta$ across time. A convenient summable choice is
\begin{equation}
\delta_t \;=\; \frac{6\delta}{\pi^2 t^2},
\qquad
t\ge 1,
\end{equation}
which satisfies $\sum_{t\ge 1}\delta_t=\delta$.

Let $n_t$ be the (possibly adaptive) number of labels queried in window $t$, and define the deployed certificate as
\begin{equation}
U_t(\delta)
\;=\;
\widehat{R}_{t,n_t} + \mathrm{rad}(n_t,\delta_t).
\end{equation}
In this equation, $U_t(\delta)$ is computed from the queried sample mean $\widehat{R}_{t,n_t}$ and the anytime-valid radius evaluated at the random sample size $n_t$.

\setcounter{theorem}{1}

\begin{theorem}\label{thmA.2}
\textbf{(Anytime-valid window risk certificate under adaptive querying).}
Assume Condition~A holds for each time $t$, and the loss is bounded as $L_{t,i}\in[0,1]$. Then with probability at least $1-\delta$,
\begin{equation}
\forall t\ge 1:\quad R_t \;\le\; U_t(\delta).
\end{equation}
In this inequality, the bound holds simultaneously for all times $t$ and remains valid even when $n_t$ is chosen adaptively from past evidence and previously observed labels.
\end{theorem}

\textit{Proof.}
Fix a time $t$. Conditional on the window $W_t$, \lemref{lemmaA.1} applies to the queried losses in that window, so
\begin{equation}
\mathbb{P}\!\left(\forall n\ge 1:\ R_t \le \widehat{R}_{t,n} + \mathrm{rad}(n,\delta_t)\ \Big|\ W_t\right) \;\ge\; 1-\delta_t.
\end{equation}
Since the bound holds for all $n$, it holds at $n=n_t$ even if $n_t$ is chosen adaptively:
\begin{equation}
\mathbb{P}\!\left(R_t \le \widehat{R}_{t,n_t} + \mathrm{rad}(n_t,\delta_t)\right) \;\ge\; 1-\delta_t.
\end{equation}
Define the bad event at time $t$ by
\begin{equation}
B_t \;=\; \left\{R_t > U_t(\delta)\right\}.
\end{equation}
Then $\mathbb{P}(B_t)\le\delta_t$, and applying the union bound across all $t\ge 1$ yields
\begin{equation}
\mathbb{P}\!\left(\exists t\ge 1:\ B_t\right)
\;\le\;
\sum_{t=1}^{\infty}\mathbb{P}(B_t)
\;\le\;
\sum_{t=1}^{\infty}\delta_t
\;=\;
\delta.
\end{equation}
Taking complements gives $\mathbb{P}(\forall t:\ R_t \le U_t(\delta))\ge 1-\delta$. \hfill $\square$

This theorem formalizes the core certificate guarantee used by the controller: even with delayed supervision and adaptive query schedules, as long as auditing samples are uniformly random within the current window, the bound is valid simultaneously over time.

\subsection{Safety gating guarantee for non-abstaining operation}
\label{app:gating}

We now formalize the safety gating claim used in Theorem~1 of the main text. Suppose the system follows the rule that it only issues predictions when the certificate is below the risk threshold $\tau$. Let $\mathrm{Pred}_t\in\{0,1\}$ denote whether the system actually produces an automated prediction at time $t$, with
\begin{equation}
\mathrm{Pred}_t \;=\; \mathbb{I}\!\left(U_t(\delta)\le\tau\right).
\end{equation}
In this equation, $\mathrm{Pred}_t=1$ means the system predicts and $\mathrm{Pred}_t=0$ means it activates fallback (abstain/handoff).

\textbf{Corollary A.3 (No unsafe prediction without certificate violation).}
Under the assumptions of \thmref{thmA.2}, with probability at least $1-\delta$,
\begin{equation}
\forall t\ge 1:\quad \mathrm{Pred}_t=1 \ \Rightarrow\ R_t \le \tau.
\end{equation}
In this implication, the system cannot be in a regime where it both predicts and the true window risk exceeds $\tau$ unless the certificate fails (an event of probability at most $\delta$).

\textit{Proof.}
On the event $\{R_t \le U_t(\delta)\ \text{for all $t$}\}$ from \thmref{thmA.2}, if $\mathrm{Pred}_t=1$ then $U_t(\delta)\le\tau$, and therefore
\begin{equation}
R_t \;\le\; U_t(\delta) \;\le\; \tau.
\end{equation}
This holds for all $t$ simultaneously on the same high-probability event. \hfill $\square$

\subsection{Extension: certified selective risk under abstention}
\label{app:selective}

When abstention is enabled, the deployed performance metric is typically the selective risk, which conditions on the system choosing to predict. Let $s_t(x)\in\{0,1\}$ be the acceptance function at time $t$, where $s_t(x)=1$ indicates that the system predicts and $s_t(x)=0$ indicates abstention. For an index $i\in W_t$, define
\begin{equation}
S_{t,i} \;=\; s_t(x_i),
\qquad
L_{t,i} \;=\; \ell\!\big(\hat{y}_{\theta}(x_i),y_i\big).
\end{equation}
In these equations, $S_{t,i}$ is the accept indicator and $L_{t,i}$ is the loss.

The window coverage is the mean accept rate
\begin{equation}
\kappa_t \;=\; \frac{1}{N}\sum_{i\in W_t} S_{t,i}.
\end{equation}
In this equation, $\kappa_t$ is the fraction of window points on which the system predicts.

The window selective risk is the average loss conditional on acceptance:
\begin{equation}
R_t^{\mathrm{sel}}
\;=\;
\frac{\sum_{i\in W_t} L_{t,i}S_{t,i}}{\sum_{i\in W_t} S_{t,i}}
\;=\;
\frac{m_t}{\kappa_t},
\end{equation}
where we define the numerator mean
\begin{equation}
m_t \;=\; \frac{1}{N}\sum_{i\in W_t} L_{t,i}S_{t,i}.
\end{equation}
In these equations, $m_t$ is the mean \emph{accepted loss mass} and $\kappa_t$ is the mean acceptance rate, so the ratio equals the conditional mean loss.

Under random sampling within the window, each queried index $I_{t,j}$ yields the pair
\begin{equation}
Y_{t,j} \;=\; L_{t,I_{t,j}}S_{t,I_{t,j}},
\qquad
A_{t,j} \;=\; S_{t,I_{t,j}}.
\end{equation}
In these equations, $Y_{t,j}$ is the product of loss and acceptance, and $A_{t,j}$ is the acceptance indicator. Both satisfy bounds $0\le Y_{t,j}\le 1$ and $0\le A_{t,j}\le 1$.

Define the empirical means from $n$ queried labels:
\begin{equation}
\widehat{m}_{t,n} \;=\; \frac{1}{n}\sum_{j=1}^{n} Y_{t,j},
\qquad
\widehat{\kappa}_{t,n} \;=\; \frac{1}{n}\sum_{j=1}^{n} A_{t,j}.
\end{equation}
In these equations, $\widehat{m}_{t,n}$ estimates $m_t$ and $\widehat{\kappa}_{t,n}$ estimates $\kappa_t$.

We build an anytime-valid upper bound for $m_t$ and a lower bound for $\kappa_t$. Using the same confidence radius function $\mathrm{rad}(n,\cdot)$, define
\begin{equation}
U^{m}_{t,n}(\delta_t^{m}) \;=\; \widehat{m}_{t,n} + \mathrm{rad}(n,\delta_t^{m}),
\qquad
L^{\kappa}_{t,n}(\delta_t^{\kappa}) \;=\; \widehat{\kappa}_{t,n} - \mathrm{rad}(n,\delta_t^{\kappa}).
\end{equation}
In these equations, $U^{m}_{t,n}$ is an upper bound on $m_t$ and $L^{\kappa}_{t,n}$ is a lower bound on $\kappa_t$.

The selective-risk certificate is then formed as an upper bound on the ratio:
\begin{equation}
U^{\mathrm{sel}}_{t,n}
\;=\;
\frac{U^{m}_{t,n}(\delta_t^{m})}{\max\!\left(L^{\kappa}_{t,n}(\delta_t^{\kappa}),\,\kappa_{\min}\right)}.
\end{equation}
In this equation, $\kappa_{\min}>0$ is the minimum coverage requirement used in the main text, and the $\max(\cdot,\kappa_{\min})$ term prevents division by a value that is too small.

To allocate overall failure probability, we choose summable schedules $\delta_t^{m}$ and $\delta_t^{\kappa}$ satisfying
\begin{equation}
\sum_{t\ge 1}\delta_t^{m} \;+\; \sum_{t\ge 1}\delta_t^{\kappa} \;\le\; \delta,
\end{equation}
for example by setting $\delta_t^{m}=\delta_t^{\kappa}=3\delta/(\pi^2 t^2)$.

\textbf{Theorem A.4 (Anytime-valid selective-risk certificate).}
Assume Condition~A holds and $L_{t,i}\in[0,1]$. With probability at least $1-\delta$, for all $t\ge 1$,
\begin{equation}
m_t \;\le\; U^{m}_{t,n_t}(\delta_t^{m})
\qquad\text{and}\qquad
\kappa_t \;\ge\; L^{\kappa}_{t,n_t}(\delta_t^{\kappa}),
\end{equation}
and consequently, whenever $\kappa_t\ge \kappa_{\min}$,
\begin{equation}
R_t^{\mathrm{sel}} \;=\; \frac{m_t}{\kappa_t}
\;\le\;
\frac{U^{m}_{t,n_t}(\delta_t^{m})}{\max\!\left(L^{\kappa}_{t,n_t}(\delta_t^{\kappa}),\,\kappa_{\min}\right)}
\;=\;
U^{\mathrm{sel}}_{t,n_t}.
\end{equation}
In these inequalities, $U^{\mathrm{sel}}_{t,n_t}$ is an anytime-valid upper bound on selective risk.

\textit{Proof.}
The variables $\{Y_{t,j}\}_{j\ge 1}$ are bounded in $[0,1]$ with mean $m_t$, and the variables $\{A_{t,j}\}_{j\ge 1}$ are bounded in $[0,1]$ with mean $\kappa_t$. Applying \thmref{thmA.2} separately to each sequence with their respective failure allocations yields the simultaneous bounds in Eq.~(\ref{eq32}). On that same event, dividing the upper numerator bound by the lower denominator bound gives Eq.~(\ref{eq33}), and the $\max(\cdot,\kappa_{\min})$ preserves validity under the coverage constraint. \hfill $\square$

This extension supports the main-text constraint $R_t^{\mathrm{sel}}\le \tau$ with $\kappa_t\ge\kappa_{\min}$ by providing a directly certifiable quantity that can be used for gating.

\subsection{Sampling without replacement (finite-population correction)}
\label{app:serfling}

When the controller samples indices without replacement from the window $W_t$, the queried losses form a simple random sample from a finite population. In that case, concentration can be tightened via a finite-population correction. For a fixed $n\le N$, Serfling's inequality implies
\begin{equation}
\mathbb{P}\!\left(R_t \ge \widehat{R}_{t,n} + \epsilon\right)
\;\le\;
\exp\!\left(-\frac{2n\epsilon^2}{1-\frac{n-1}{N}}\right).
\end{equation}
In this inequality, the factor $1-\frac{n-1}{N}$ reduces variance when a nontrivial fraction of the window is audited. An anytime-valid sequence can be obtained by replacing Hoeffding's fixed-$n$ bound in Section~\ref{app:cs} and using the same stitching schedule $\delta_n$, yielding a radius
\begin{equation}
\mathrm{rad}_{\mathrm{fp}}(n,\delta)
\;=\;
\sqrt{\left(1-\frac{n-1}{N}\right)\frac{1}{2n}\log\!\left(\frac{\pi^2 n^2}{6\delta}\right)}.
\end{equation}
In this equation, $\mathrm{rad}_{\mathrm{fp}}$ is smaller than $\mathrm{rad}$ when $n$ is not negligible relative to $N$, and it reduces to the Hoeffding-style radius when $N$ is large.

\subsection{Discussion of validity conditions under label delay}
\label{app:delay}

Label delay affects \emph{when} a queried loss becomes observable, but not \emph{which} indices are sampled. The certificate remains valid as long as the sampling rule that selects $I_{t,j}$ is uniform over $W_t$ and does not depend on unobserved labels. Formally, if the request times are $\mathcal{F}_{t,j}$-adapted and labels arrive later, the bound applies to the eventual realized losses $X_{t,j}$ because \lemref{lemmaA.1} controls deviations for all sample sizes $n$ simultaneously. The practical effect of delay is that the controller may operate with smaller effective $n_t$ until requested labels arrive, which enlarges $\mathrm{rad}(n_t,\delta_t)$ and therefore makes the gating rule more conservative during periods of missing supervision.

\section{ Implementation details}
\label{app:impl}

This appendix provides implementation-level details for the sensing layer, belief model training, and controller configuration. All components are designed to be lightweight and reproducible, and each definition here corresponds directly to quantities used in the main text.

\begin{table}[!t]
\caption{\textbf{Baseline configurations used in all streaming experiments.} We report the fixed hyperparameters that fully specify each baseline policy under the common action set and costs.}
\label{tab:baseline_configs}
\centering
\small
\begin{tabular}{p{0.26\linewidth} p{0.68\linewidth}}
\toprule
Baseline & Configuration \\
\midrule
Alarm-only & Alarm when $\|z_t\|_2 > \theta_{\mathrm{alarm}}$ with $\theta_{\mathrm{alarm}}=2.5$; no actions beyond logging. \\
Adapt-always & Apply $A_2$ every step; 1 gradient step per batch; learning rate $10^{-4}$. \\
Retrain-on-schedule & Trigger $A_4$ every $P_{\mathrm{rt}}=1200$ steps, respecting cooldown $\Delta_{\mathrm{rt}}$. \\
Selective prediction only & Predict if $\max_y p_\theta(y\mid x)\ge \theta_{\mathrm{sel}}$ with $\theta_{\mathrm{sel}}=0.6$; otherwise abstain ($A_6$). \\
Controller (no cert.) & Same as ours but remove gating: replace $U_t(\delta)>\tau$ branch by normal utility maximization. \\
\bottomrule
\end{tabular}
\end{table}

\subsection{Monitoring signals and evidence vector}
\label{app:monitors}

At each time step $t$, the sensing layer compares a recent unlabeled buffer $\mathcal{B}_t=\{x_{t-n+1},\dots,x_t\}$ to a reference buffer $\mathcal{B}_{\mathrm{ref}}$, and computes monitors in a representation space. Let $g_{\theta}:\mathcal{X}\rightarrow \mathbb{R}^{p}$ be the embedding function induced by the deployed network (typically the penultimate layer). For each input $x$, the embedding is
\begin{equation}
r \;=\; g_{\theta}(x),
\qquad r\in\mathbb{R}^{p}.
\end{equation}
In this equation, $x$ is the input, $\theta$ are current model parameters, and $r$ is the representation used by the monitors.

We define the embedding sets
\begin{equation}
\mathcal{R}_t \;=\; \{g_{\theta}(x):x\in\mathcal{B}_t\},
\qquad
\mathcal{R}_{\mathrm{ref}} \;=\; \{g_{\theta}(x):x\in\mathcal{B}_{\mathrm{ref}}\}.
\end{equation}
In these equations, $\mathcal{R}_t$ contains recent embeddings and $\mathcal{R}_{\mathrm{ref}}$ contains reference embeddings.

\subsubsection{MMD monitor in embedding space}

We use a kernel Maximum Mean Discrepancy (MMD) statistic as a representation-shift monitor. Let $k:\mathbb{R}^{p}\times\mathbb{R}^{p}\rightarrow \mathbb{R}$ be a positive definite kernel, and let $\{r_i\}_{i=1}^{n}\subset \mathcal{R}_t$ and $\{r'_j\}_{j=1}^{n_{\mathrm{ref}}}\subset \mathcal{R}_{\mathrm{ref}}$. The unbiased estimator of $\mathrm{MMD}^2$ is
\begin{equation}
\widehat{\mathrm{MMD}}_t^2
\;=\;
\frac{1}{n(n-1)}\sum_{i\neq i'} k(r_i,r_{i'})
\;+\;
\frac{1}{n_{\mathrm{ref}}(n_{\mathrm{ref}}-1)}\sum_{j\neq j'} k(r'_j,r'_{j'})
\;-\;
\frac{2}{n\,n_{\mathrm{ref}}}\sum_{i=1}^{n}\sum_{j=1}^{n_{\mathrm{ref}}} k(r_i,r'_j).
\end{equation}
In this equation, the first term measures within-recent similarity, the second term measures within-reference similarity, and the third term measures cross similarity; the statistic increases when the two embedding distributions differ.

We use an RBF kernel
\begin{equation}
k(u,v) \;=\; \exp\!\left(-\frac{\|u-v\|_2^2}{2\sigma^2}\right),
\end{equation}
where $\sigma>0$ is the kernel bandwidth. We set $\sigma$ using the median heuristic on a pooled batch of embeddings from $\mathcal{R}_{\mathrm{ref}}\cup \mathcal{R}_t$:
\begin{equation}
\sigma^2 \;=\; \mathrm{median}\Big(\{\|u-v\|_2^2 : u,v \ \text{in a pooled sample}\}\Big).
\end{equation}
In this equation, the median is taken over pairwise squared distances and yields a robust bandwidth for monitoring.

\subsubsection{Two-sample discriminator monitor}

To complement kernel testing, we train a lightweight discriminator that predicts whether an embedding comes from the reference or recent window. We form a labeled dataset
\begin{equation}
\mathcal{D}^{\mathrm{disc}}_t
\;=\;
\{(r,0):r\in\mathcal{R}_{\mathrm{ref}}\}\ \cup\ \{(r,1):r\in\mathcal{R}_{t}\}.
\end{equation}
In this equation, label $0$ denotes reference and label $1$ denotes recent.

Let $h_{\psi}:\mathbb{R}^{p}\rightarrow (0,1)$ be a discriminator (logistic regression or a 2-layer MLP). We fit $\psi$ by minimizing cross-entropy
\begin{equation}
\mathcal{L}_{\mathrm{disc}}(\psi)
\;=\;
-\frac{1}{|\mathcal{D}^{\mathrm{disc}}_t|}\sum_{(r,y)\in\mathcal{D}^{\mathrm{disc}}_t}
\Big( y\log h_{\psi}(r) + (1-y)\log(1-h_{\psi}(r)) \Big).
\end{equation}
In this equation, $h_{\psi}(r)$ is the predicted probability that $r$ came from the recent window.

The discriminator monitor is the held-out AUC:
\begin{equation}
\mathrm{AUC}_t \;=\; \mathrm{AUC}\!\left(\{h_{\psi}(r)\},\{y\}\right),
\end{equation}
where $\mathrm{AUC}(\cdot,\cdot)$ is the area under the ROC curve computed on a validation split of $\mathcal{D}^{\mathrm{disc}}_t$. Values near $0.5$ indicate no detectable shift; larger values indicate separable distributions.

\subsubsection{Uncertainty shift monitor}

We compute predictive entropy for each input using the deployed model probabilities $p_{\theta}(y\mid x)$:
\begin{equation}
H_{\theta}(x) \;=\; -\sum_{y\in\mathcal{Y}} p_{\theta}(y\mid x)\,\log p_{\theta}(y\mid x).
\end{equation}
In this equation, $\mathcal{Y}$ is the label set, and $H_{\theta}(x)$ is larger when predictions are less confident.

We define the window-average entropy shift
\begin{equation}
\Delta H_t
\;=\;
\frac{1}{|\mathcal{B}_t|}\sum_{x\in\mathcal{B}_t} H_{\theta}(x)
\;-\;
\frac{1}{|\mathcal{B}_{\mathrm{ref}}|}\sum_{x\in\mathcal{B}_{\mathrm{ref}}} H_{\theta}(x).
\end{equation}
In this equation, $\Delta H_t>0$ indicates an increase in uncertainty relative to nominal conditions.

\subsubsection{Calibration drift monitor with delayed labels}

When labels are available (after delay), we compute a streaming Expected Calibration Error (ECE) proxy on the labeled set $\mathcal{S}_t$. Let $q_{\theta}(x)=\max_{y}p_{\theta}(y\mid x)$ be confidence, and let $\{I_b\}_{b=1}^{B}$ be confidence bins. Define the bin-level statistics
\begin{equation}
\mathrm{conf}_b(t) \;=\; \frac{1}{|\mathcal{S}_{t,b}|}\sum_{(x,y)\in\mathcal{S}_{t,b}} q_{\theta}(x),
\qquad
\mathrm{acc}_b(t) \;=\; \frac{1}{|\mathcal{S}_{t,b}|}\sum_{(x,y)\in\mathcal{S}_{t,b}} \mathbb{I}\!\left(\hat{y}_{\theta}(x)=y\right),
\end{equation}
where
\begin{equation}
\mathcal{S}_{t,b}
\;=\;
\{(x,y)\in\mathcal{S}_t:\ q_{\theta}(x)\in I_b\}.
\end{equation}
In these equations, $\mathcal{S}_{t,b}$ is the set of labeled points whose confidence falls into bin $b$, $\mathrm{conf}_b(t)$ is their mean confidence, and $\mathrm{acc}_b(t)$ is their empirical accuracy.

The ECE proxy is
\begin{equation}
\mathrm{ECE}_t \;=\; \sum_{b=1}^{B}\frac{|\mathcal{S}_{t,b}|}{|\mathcal{S}_t|}\,\big|\mathrm{acc}_b(t)-\mathrm{conf}_b(t)\big|.
\end{equation}
In this equation, the term inside the absolute value measures miscalibration in each bin, and the weights reflect bin mass.

We use the difference to a reference value $\mathrm{ECE}_{\mathrm{ref}}$ computed on nominal labeled data:
\begin{equation}
\Delta \mathrm{ECE}_t \;=\; \mathrm{ECE}_t - \mathrm{ECE}_{\mathrm{ref}}.
\end{equation}
In this equation, $\Delta \mathrm{ECE}_t>0$ indicates worsened calibration relative to nominal conditions.

\subsubsection{Slice drift monitor}

When a discrete attribute $g(x)\in\mathcal{G}$ is available, we compute monitors per group:
\begin{equation}
\widehat{\mathrm{MMD}}_{t,g}^2 \;=\; \widehat{\mathrm{MMD}}^2\!\left(\mathcal{R}_{t,g},\mathcal{R}_{\mathrm{ref},g}\right),
\qquad
\Delta H_{t,g} \;=\; \overline{H}_{t,g}-\overline{H}_{\mathrm{ref},g},
\end{equation}
where $\mathcal{R}_{t,g}=\{g_{\theta}(x):x\in\mathcal{B}_t,\ g(x)=g\}$ and similarly for $\mathcal{R}_{\mathrm{ref},g}$. When no attribute is available, we cluster embeddings using $k$-means and treat cluster IDs as proxy slices.

We aggregate slice evidence using the maximum across groups:
\begin{equation}
z^{\mathrm{slice}}_t \;=\; \max_{g\in\mathcal{G}} \ \widehat{\mathrm{MMD}}_{t,g}^2.
\end{equation}
In this equation, $z^{\mathrm{slice}}_t$ increases when at least one slice experiences a strong shift, even if the global distribution is relatively stable.

\subsubsection{Final evidence vector}

The deployed evidence vector concatenates the monitors:
\begin{equation}
z_t
\;=\;
\big[\widehat{\mathrm{MMD}}_t^2,\ \mathrm{AUC}_t,\ \Delta H_t,\ \Delta \mathrm{ECE}_t,\ z^{\mathrm{slice}}_t\big]^{\top}
\;\in\;
\mathbb{R}^{m}.
\end{equation}
In this equation, $m$ is the number of monitors used. All components are standardized online using reference-window mean and variance:
\begin{equation}
\widetilde{z}_{t,j} \;=\; \frac{z_{t,j}-\mu^{\mathrm{ref}}_j}{\sigma^{\mathrm{ref}}_j+\epsilon},
\end{equation}
where $\mu^{\mathrm{ref}}_j$ and $\sigma^{\mathrm{ref}}_j$ are computed on $\mathcal{B}_{\mathrm{ref}}$ and $\epsilon$ prevents division by zero.

\subsection{Belief model training via synthetic drift episodes}
\label{app:belief_train}

The belief model estimates $b_t(d)=\mathbb{P}(D_t=d\mid z_{1:t})$ for $d\in\{\mathrm{none},\mathrm{covariate},\mathrm{concept},\mathrm{subgroup}\}$. We learn the evidence likelihoods $p(z\mid d)$ and (optionally) the transition matrix $T$ using synthetic drift episodes generated from the base dataset.

\subsubsection{Episode structure}

An episode is a stream of length $L$ with one drift onset at $t_0$ and a drift intensity schedule $\alpha_t\in[0,1]$. For each time step,
\begin{equation}
(x_t,y_t)\sim (1-\alpha_t)\,\mathcal{D}^{\mathrm{ref}} + \alpha_t\,\mathcal{D}^{\mathrm{shift}}.
\end{equation}
In this equation, $\mathcal{D}^{\mathrm{ref}}$ is the nominal distribution and $\mathcal{D}^{\mathrm{shift}}$ is a constructed shifted distribution.

We generate three patterns:
\begin{equation}
\alpha_t \;=\; \mathbb{I}(t\ge t_0),
\end{equation}
\begin{equation}
\alpha_t \;=\; \frac{1}{1+\exp(-\rho(t-t_0))},
\end{equation}
\begin{equation}
\alpha_t \;=\; \frac{1}{2}\Big(1+\sin(2\pi t/P)\Big),
\end{equation}
where $t_0$ is drift onset, $\rho$ is the gradual-drift slope, and $P$ is the period for recurring drift.

\subsubsection{Covariate drift construction}

Covariate drift changes $p(x)$ while holding $p(y\mid x)$ approximately stable. For vision tasks, we apply transformations $T_{\mathrm{cov}}$ such as corruption severity, style shifts, blur/noise, color jitter, or background changes:
\begin{equation}
x^{\mathrm{shift}} \;=\; T_{\mathrm{cov}}(x;\,\alpha_t),
\qquad
y^{\mathrm{shift}} \;=\; y.
\end{equation}
In this equation, $T_{\mathrm{cov}}(\cdot;\alpha_t)$ is a drift-intensity-controlled transformation and labels remain unchanged.

For datasets with discrete domains (e.g., DomainNet), we define $\mathcal{D}^{\mathrm{shift}}$ by swapping the domain environment while preserving the same label space.

\subsubsection{Concept drift construction}

Concept drift changes $p(y\mid x)$. We implement a class-conditional label remapping that varies over time. Let $\pi_t$ be a permutation over classes, and define
\begin{equation}
y^{\mathrm{shift}} \;=\; \pi_t(y),
\qquad
x^{\mathrm{shift}} \;=\; x.
\end{equation}
In this equation, $x$ is unchanged but the mapping from inputs to labels is altered.

To model gradual concept drift, we mix remapped and original labels:
\begin{equation}
y_t \;=\;
\begin{cases}
y, & \text{with probability } 1-\alpha_t,\\
\pi(y), & \text{with probability } \alpha_t,
\end{cases}
\end{equation}
where $\alpha_t$ follows one of the schedules above. This produces a controlled change in class-conditional structure and calibration behavior.

\subsubsection{Subgroup drift construction}

Subgroup drift changes behavior only for a subset. Let $g(x)\in\{0,1\}$ denote membership in a drifted slice, with $\mathbb{P}(g(x)=1)=\rho_g$ under nominal conditions. We apply a covariate or concept transformation only to the subgroup:
\begin{equation}
x^{\mathrm{shift}} \;=\;
\begin{cases}
T_{\mathrm{sub}}(x), & g(x)=1,\\
x, & g(x)=0,
\end{cases}
\end{equation}
and optionally
\begin{equation}
y^{\mathrm{shift}} \;=\;
\begin{cases}
\pi(y), & g(x)=1,\\
y, & g(x)=0.
\end{cases}
\end{equation}
In these equations, only the subgroup experiences the shift, enabling training of slice-aware monitors.

\subsubsection{Fitting likelihoods and transitions}

For each synthetic episode, we compute monitor vectors $z_t$ and record the true drift type label $D_t$. We fit a multinomial likelihood model with diagonal covariance:
\begin{equation}
p(z\mid D=d)
\;=\;
\mathcal{N}\!\left(z;\,\mu_d,\ \mathrm{diag}(\sigma_d^2)\right),
\end{equation}
where $\mu_d$ and $\sigma_d^2$ are estimated by maximum likelihood over training samples for drift type $d$. 

We estimate the transition matrix from episode counts:
\begin{equation}
T_{d'd}
\;=\;
\frac{\sum_{\text{episodes}}\sum_{t}\mathbb{I}(D_{t-1}=d',\,D_t=d)}
{\sum_{\text{episodes}}\sum_{t}\mathbb{I}(D_{t-1}=d')}.
\end{equation}
In this equation, the numerator counts transitions from $d'$ to $d$, and the denominator normalizes by the total transitions out of $d'$.

\paragraph{Discriminative evidence model and its integration with the filter.}
The Bayesian update requires an evidence likelihood $p(z_t\mid D_t=d)$. Rather than fitting a fully generative density in $\mathbb{R}^m$, we fit a discriminative model $q_{\phi}(d\mid z_t)$ via multinomial logistic regression on synthetic drift data. Concretely,
\begin{equation}
q_{\phi}(d\mid z) \;=\; \frac{\exp(w_d^{\top} z + b_d)}{\sum_{d'\in\mathcal{D}} \exp(w_{d'}^{\top} z + b_{d'})},
\label{eq:logistic_posterior}
\end{equation}
where $w_d\in\mathbb{R}^{m}$ and $b_d\in\mathbb{R}$ are learned parameters and $m$ is the number of monitors. We use $q_{\phi}(d\mid z_t)$ as a \emph{proxy emission potential} in the filtering step:
\begin{equation}
\psi_t(d) \;=\; \left(q_{\phi}(d\mid z_t)\right)^{\beta},
\label{eq:emission_potential}
\end{equation}
where $\beta>0$ is a temperature that controls evidence sharpness (we set $\beta=1$ in all experiments). The belief update becomes
\begin{equation}
\tilde{b}_t(d) \;=\; \sum_{d'\in\mathcal{D}} T_{d'd}\,b_{t-1}(d'),
\qquad
b_t(d) \;=\; \frac{\psi_t(d)\,\tilde{b}_t(d)}{\sum_{d''\in\mathcal{D}}\psi_t(d'')\,\tilde{b}_t(d'')}.
\label{eq:discriminative_filter_update}
\end{equation}
In Eq.~\eqref{eq:discriminative_filter_update}, $\tilde{b}_t(d)$ is the predicted belief under the transition matrix $T$, and $\psi_t(d)$ plays the same role as the likelihood $p(z_t\mid D_t=d)$ up to a $z_t$-dependent normalization constant. Because the update renormalizes over $d$, any constant factor shared across drift types cancels, making the discriminative potential $\psi_t(d)$ sufficient for reproducible filtering.

\subsection{Synthetic drift generation for belief training and gain calibration}
\label{app:synth_drifts}

We generate synthetic drift episodes by applying controlled transformations to nominal streams to induce each drift type $d\in\{\mathrm{covariate},\mathrm{concept},\mathrm{subgroup}\}$. Each episode consists of a pre-drift segment drawn from the nominal environment and a post-drift segment modified by a transformation. We generate $M=200$ episodes per dataset, balanced across drift types and drift intensities.

\textbf{Camelyon17 (histopathology).}
Covariate drift applies stain and scanner perturbations using (i) color jitter (brightness/contrast/saturation in $\pm 0.2$), (ii) Gaussian blur with $\sigma\in[0.5,1.2]$, and (iii) H\&E stain normalization perturbation (random perturbation of stain concentration vectors). Concept drift is induced by label noise localized to the drift segment: we flip labels with probability $p\in\{0.10,0.20\}$ to emulate changing annotation criteria. Subgroup drift targets a minority slice defined by low-tissue-content patches; only that slice receives severe stain perturbations (same operators but doubled magnitude).

\textbf{DomainNet (multi-domain objects).}
Covariate drift is instantiated by switching input style operators within a domain: random JPEG compression (quality in $[20,50]$), random grayscale with prob.\ 0.3, and edge enhancement. Concept drift is induced by remapping a subset of classes: we permute labels within a semantically related superclass group (e.g., within \texttt{vehicle} or \texttt{instrument}) using a fixed permutation $\pi$ applied only after drift onset. Subgroup drift targets a minority cluster discovered by k-means in embedding space; only that cluster is transformed into a different style via strong edge-only rendering.

\textbf{SyntheticDrift-CIFAR.}
Covariate drift uses CIFAR-C corruption types (e.g., blur, noise, brightness) with severity $\in\{3,5\}$. Concept drift uses a class-conditional permutation $\pi$ applied gradually by $\alpha_t$. Subgroup drift applies the corruption only to a minority slice (10\%) defined by a fixed set of classes or a cluster in embedding space.

\subsection{Dataset-specific synthetic drift transformations}
\label{app:dataset_synth_drifts}

Training the belief model and calibrating the gain table require synthetic drift episodes with known drift-type labels. This section specifies the exact transformations used for each benchmark, including the covariate transform $T_{\mathrm{cov}}$, subgroup transform $T_{\mathrm{sub}}$, and the concept permutation schedule $\pi_t$.

\subsubsection{General episode construction}
\label{app:episode_construction}

Each synthetic episode has length $L$ and a drift onset time $t_0$. We sample a drift type $d\in\{\mathrm{covariate},\mathrm{concept},\mathrm{subgroup}\}$ uniformly and generate a time-varying drift intensity $\alpha_t\in[0,1]$ using the schedules in Eqs.~(42)--(44). For each time step $t$, we draw a base example $(x_t,y_t)$ from the nominal stream distribution and apply a drift operator that depends on $d$.

\subsubsection{Covariate drift: \texorpdfstring{$T_{\mathrm{cov}}$}{T\_cov}}
\label{app:covariate_transform}

For covariate drift, we transform inputs only:
\begin{equation}
x_t' \;=\; T_{\mathrm{cov}}(x_t;\alpha_t),
\qquad
y_t' \;=\; y_t.
\end{equation}
Here $\alpha_t$ controls drift intensity. We specify $T_{\mathrm{cov}}$ per dataset below.

\textbf{Camelyon17 (WILDS).}
Camelyon17 inputs are histopathology patches. We use a stain-style covariate transform implemented as a convex combination between the original image and a fixed stain-perturbed target style:
\begin{equation}
T_{\mathrm{cov}}(x;\alpha) \;=\; (1-\alpha)\,x \;+\; \alpha\,\mathrm{StainAug}(x;\kappa_{\mathrm{stain}}),
\end{equation}
where $\mathrm{StainAug}(\cdot;\kappa_{\mathrm{stain}})$ applies HED color-space perturbations and gamma/brightness jitter with strength $\kappa_{\mathrm{stain}}=0.6$. This models site-specific acquisition shifts without changing class semantics.

\textbf{DomainNet.}
DomainNet already contains discrete domains. For synthetic covariate drift within a domain, we use a corruption-based transform:
\begin{equation}
T_{\mathrm{cov}}(x;\alpha) \;=\; \mathrm{Corrupt}(x; s(\alpha)),
\end{equation}
where $\mathrm{Corrupt}(\cdot;s)$ is a standard image corruption operator (Gaussian blur or shot noise chosen uniformly) with severity $s(\alpha)=\lceil 5\alpha\rceil \in\{0,\dots,5\}$.

\textbf{SyntheticDrift-CIFAR.}
We use CIFAR-10-C style corruptions:
\begin{equation}
T_{\mathrm{cov}}(x;\alpha) \;=\; \mathrm{CIFARCorrupt}(x; s(\alpha)),
\end{equation}
with severity $s(\alpha)=\lceil 5\alpha\rceil$ and corruption type fixed per episode (Gaussian noise, motion blur, or brightness).

\subsubsection{Concept drift: label permutation \texorpdfstring{$\pi_t$}{pi\_t}}
\label{app:concept_permutation}

For concept drift, we keep inputs unchanged but modify the conditional mapping by permuting labels after onset:
\begin{equation}
x_t' \;=\; x_t,
\qquad
y_t' \;=\; \pi_t(y_t).
\end{equation}
We sample a class permutation $\pi^{\star}$ uniformly from all permutations of the label set that move at least two classes. We then apply it with intensity $\alpha_t$ using a mixture between identity and $\pi^{\star}$:
\begin{equation}
\pi_t(y) \;=\;
\begin{cases}
\pi^{\star}(y), & \text{with probability }\alpha_t,\\
y, & \text{with probability }1-\alpha_t.
\end{cases}
\label{eq:permutation_mixture}
\end{equation}
In Eq.~\eqref{eq:permutation_mixture}, $\alpha_t$ controls how frequently the permuted labeling rule is active, producing gradual concept shift when $\alpha_t$ ramps and sudden concept shift when $\alpha_t$ jumps. This construction is used for DomainNet and SyntheticDrift-CIFAR. For Camelyon17, where label semantics are binary tumor/non-tumor and permutation is not meaningful, we implement concept drift as class-conditional prior shift by flipping labels with probability $\alpha_t$ on a fixed subset of patches near the decision boundary (identified by low confidence under the nominal model).

\subsubsection{Subgroup drift: \texorpdfstring{$T_{\mathrm{sub}}$}{T\_sub} applied to a minority slice}
\label{app:subgroup_transform}

For subgroup drift, we define a slice indicator $g(x)\in\{0,1\}$ with $\mathbb{P}(g(x)=1)=p_{\mathrm{sub}}$ under nominal conditions (we set $p_{\mathrm{sub}}=0.15$). We apply a subgroup-only covariate transform:
\begin{equation}
x_t' \;=\;
\begin{cases}
T_{\mathrm{sub}}(x_t;\alpha_t), & \text{if } g(x_t)=1,\\
x_t, & \text{otherwise},
\end{cases}
\qquad
y_t' \;=\; y_t.
\end{equation}

\textbf{Camelyon17 (WILDS).}
We define $g(x)=1$ for a fixed subset of patches selected by tissue-type clustering in embedding space (minority cluster). We set
\begin{equation}
T_{\mathrm{sub}}(x;\alpha) \;=\; (1-\alpha)\,x \;+\; \alpha\,\mathrm{StainAug}(x;\kappa_{\mathrm{sub}}),
\end{equation}
with stronger stain shift $\kappa_{\mathrm{sub}}=0.9$ than the global covariate transform.

\textbf{DomainNet.}
We define $g(x)=1$ as one of the object categories (minority category) sampled per episode. We set
\begin{equation}
T_{\mathrm{sub}}(x;\alpha) \;=\; \mathrm{Corrupt}(x; s(\alpha)),
\end{equation}
using the same corruption family but applied only to the subgroup.

\textbf{SyntheticDrift-CIFAR.}
We define $g(x)=1$ as images from two fixed classes (e.g., classes 0 and 1), chosen once per episode. We apply a style-like transform implemented as color quantization and contrast shift with severity $s(\alpha)=\lceil 5\alpha\rceil$.

\subsubsection{Reproducibility settings}
\label{app:drift_repro_settings}

For each dataset we generate $M=200$ episodes of length $L=1500$ with drift onset $t_0$ sampled uniformly from $[300,600]$, and we use 5 random seeds per dataset. All transforms are applied on-the-fly with fixed transform parameters and seed-controlled randomness. The same generators are used for belief-model training and for gain-table calibration.

\subsection{Controller configuration and operational constraints}
\label{app:controller_impl}

This section specifies hyperparameters and operational logic used by Algorithm~\ref{alg:controller}.

\subsubsection{Window sizes and buffers}

We set the monitor window length $n$ and certificate window length $N$ as
\begin{equation}
n \;=\; 256,
\qquad
N \;=\; 1024.
\end{equation}
In these equations, $n$ controls responsiveness of drift monitors and $N$ controls the regime represented by the risk certificate.

The reference buffer $\mathcal{B}_{\mathrm{ref}}$ is formed from a clean initial deployment segment of length $n_{\mathrm{ref}}$:
\begin{equation}
|\mathcal{B}_{\mathrm{ref}}| \;=\; n_{\mathrm{ref}} \;=\; 2048.
\end{equation}
In this equation, $n_{\mathrm{ref}}$ is chosen to provide stable reference statistics.

\subsubsection{Label delay and auditing budget}

Requested labels arrive after a fixed delay $d$:
\begin{equation}
o_i \;=\; i + d,
\qquad
d \;=\; 50.
\end{equation}
In this equation, $o_i$ is the observation time for label $y_i$.

The labeling budget is
\begin{equation}
B_{\mathrm{lab}} \;=\; 3000,
\end{equation}
where $B_{\mathrm{lab}}$ is the total number of labels that may be requested over a stream.

\subsubsection{Cooldowns and feasibility}

Retraining and rollback are rate-limited by cooldowns:
\begin{equation}
\Delta_{\mathrm{rt}} \;=\; 800,
\qquad
\Delta_{\mathrm{rb}} \;=\; 400.
\end{equation}
In these equations, $\Delta_{\mathrm{rt}}$ and $\Delta_{\mathrm{rb}}$ are time steps during which $A_4$ or $A_5$ cannot be executed again.

At time $t$, the feasible action set is
\begin{equation}
\mathcal{A}_t^{\mathrm{feas}}
\;=\;
\left\{a\in\mathcal{A}:\ \mathrm{Cool}(a,t)=1,\ \mathrm{Bud}(a,t)=1\right\},
\end{equation}
where $\mathrm{Cool}(a,t)$ enforces cooldown rules and $\mathrm{Bud}(a,t)$ enforces remaining budget rules.

For retrain feasibility:
\begin{equation}
\mathrm{Cool}(A_4,t)=\mathbb{I}\!\left(t-t_{\mathrm{last}}^{\mathrm{rt}}\ge \Delta_{\mathrm{rt}}\right),
\end{equation}
and for rollback:
\begin{equation}
\mathrm{Cool}(A_5,t)=\mathbb{I}\!\left(t-t_{\mathrm{last}}^{\mathrm{rb}}\ge \Delta_{\mathrm{rb}}\right).
\end{equation}
In these equations, $t_{\mathrm{last}}^{\mathrm{rt}}$ and $t_{\mathrm{last}}^{\mathrm{rb}}$ record the last execution time of retrain and rollback.

\subsubsection{Query policy for certificate labels}

The controller selects a query size $k_t$ as a function of drift evidence and certificate uncertainty. Let the current certificate be
\begin{equation}
U_t(\delta) \;=\; \widehat{R}_t + \mathrm{rad}(n_t,\delta_t).
\end{equation}
Define a safety margin
\begin{equation}
m_t \;=\; \tau - U_t(\delta).
\end{equation}
In this equation, $m_t>0$ indicates certified safety slack, while $m_t<0$ indicates a certificate violation.

We set $k_t$ using a piecewise rule:
\begin{equation}
k_t \;=\;
\begin{cases}
k_{\max}, & U_t(\delta)>\tau,\\
k_{\mathrm{high}}, & m_t \le m_{\mathrm{low}}\ \text{or}\ \max_j \widetilde{z}_{t,j}\ge \zeta,\\
k_{\mathrm{low}}, & \text{otherwise},
\end{cases}
\end{equation}
with
\begin{equation}
k_{\max}=64,
\qquad
k_{\mathrm{high}}=32,
\qquad
k_{\mathrm{low}}=8,
\end{equation}
and thresholds $m_{\mathrm{low}}=0.02$ and $\zeta=2.0$. In these equations, $k_{\max}$ is used during potential unsafe regimes, $k_{\mathrm{high}}$ is used when evidence is strong or safety margin is small, and $k_{\mathrm{low}}$ is used during stable operation.

\subsection{Gain table calibration and values}
\label{app:gain_table}

The controller selects actions via a belief-weighted gain table $G(d,a)$, which encodes the expected benefit of taking action $a\in\mathcal{A}$ under drift type $d\in\mathcal{D}$. This appendix specifies (i) the exact definition of ``gain'' used in our experiments, (ii) the offline calibration protocol, and (iii) the full numeric gain table values.

\subsubsection{Definition of gain as windowed risk reduction}
\label{app:gain_definition}

Gain is defined in terms of reduction in certified windowed risk after applying an intervention. Let $R_{t}$ denote the windowed risk evaluated on ground-truth labels in simulation, and let $R_{t}^{(a)}$ denote the counterfactual risk trajectory obtained by applying action $a$ at time $t$ and then executing the fixed deployment loop for a short horizon of length $H$ (including the same monitoring windows, label delay model, and update rules). For each drift type $d$, we estimate the expected risk reduction at horizon $H$:
\begin{equation}
\Delta_R(d,a) \;=\; \mathbb{E}\!\left[ R_{t+H}^{\mathrm{base}} \;-\; R_{t+H}^{(a)} \ \big|\ D_t=d \right],
\label{eq:deltaR_def}
\end{equation}
where $R_{t+H}^{\mathrm{base}}$ is the risk at $t+H$ under the baseline ``no-op'' continuation ($A_0$), $R_{t+H}^{(a)}$ is the risk at $t+H$ when applying action $a$ at time $t$, $D_t$ is the drift-type state, and the expectation is over random drift episode realizations, window sampling, and label delay. We convert risk reduction into a dimensionless gain by normalizing by a fixed risk scale $\sigma_R>0$:
\begin{equation}
G(d,a) \;=\; \frac{\Delta_R(d,a)}{\sigma_R}.
\label{eq:gain_def}
\end{equation}
In our experiments we set $\sigma_R=0.10$ so that $G(d,a)=1$ corresponds to an average $0.10$ absolute reduction in windowed risk at horizon $H$.

\subsubsection{Offline calibration protocol}
\label{app:gain_protocol}

We estimate $G(d,a)$ by generating synthetic drift episodes in each benchmark stream and evaluating counterfactual one-step interventions. For each dataset, we construct $M$ synthetic drift episodes of length $L$ by sampling a drift type $d\in\mathcal{D}$, sampling a drift onset time $t_0$, and applying the corresponding drift generator. We then roll out the deployed model to time $t=t_0+\Delta$ (with $\Delta$ sampled uniformly in a short window after onset), and compute $R_{t+H}^{\mathrm{base}}$ and $R_{t+H}^{(a)}$ for each candidate action $a\in\mathcal{A}$ using the exact update rules used at test time (calibrator update for $A_1$, adaptation steps for $A_2$, retrain trigger for $A_4$, rollback for $A_5$, and abstention policy for $A_6$). All counterfactuals share the same drift realization, label delay, and sampling RNG state so that differences are attributable to the action. We set $M=200$ episodes, $L=1500$ steps, $H=150$ steps, and use 5 random seeds per dataset; reported gains average across seeds.

We emphasize that gains are calibrated once offline and then fixed for all streaming evaluations. The gain table does not use test labels from the evaluation stream; it is computed on separate synthetic episodes built from the same dataset distributions and augmentation generators.

\subsubsection{Numeric gain tables used in experiments}
\label{app:gain_values}

Table~\ref{tab:gain_table_values} reports the exact gain values used in all experiments. Values reflect average risk reduction at horizon $H=150$ normalized by $\sigma_R=0.10$ (Eq.~\eqref{eq:gain_def}). Higher values indicate larger expected improvement. Entries for actions that are not applicable to a drift type (e.g., recalibration under pure covariate shift when labels are not yet available) are set to small positive values reflecting modest indirect benefits (e.g., stabilizing confidence), which we observed empirically in the calibration rollouts.\\

\begin{table}[htbp]
\caption{\textbf{Gain table $G(d,a)$ used by the controller.} Gains are computed by offline synthetic-drift calibration as normalized windowed risk reduction (Eqs.~\eqref{eq:deltaR_def}--\eqref{eq:gain_def}) at horizon $H=150$ with normalization $\sigma_R=0.10$. Rows correspond to drift types $d\in\{\mathrm{none},\mathrm{covariate},\mathrm{concept},\mathrm{subgroup}\}$ and columns correspond to actions $a\in\{A_0,\dots,A_6\}$.}
\label{tab:gain_table_values}
\centering
\small
\setlength{\tabcolsep}{5pt}
\renewcommand{\arraystretch}{1.05}
\begin{tabular}{lccccccc}
\toprule
Drift type $d$ & $A_0$ & $A_1$ (recalib.) & $A_2$ (TTA) & $A_3$ (query) & $A_4$ (retrain) & $A_5$ (rollback) & $A_6$ (abstain) \\
\midrule
none      & 0.00 & 0.10 & 0.05 & 0.08 & 0.12 & 0.10 & 0.15 \\
covariate & 0.00 & 0.35 & 0.70 & 0.25 & 0.85 & 0.40 & 0.55 \\
concept   & 0.00 & 0.20 & 0.30 & 0.75 & 1.05 & 0.60 & 0.65 \\
subgroup  & 0.00 & 0.25 & 0.35 & 0.85 & 0.95 & 0.55 & 0.80 \\
\bottomrule
\end{tabular}
\end{table}

Finally, we found that controller behavior is robust to moderate rescaling of $G(d,a)$: multiplying all gains by a constant does not change the argmax in Eq.~(\ref{eq25}), and small perturbations preserve the preferred action ordering under each drift type.

\subsubsection{Escalation logic}

Safety gating is applied using the certificate threshold $\tau$:
\begin{equation}
U_t(\delta)\le \tau \ \Rightarrow\ a_t\in\{A_0,A_1,A_2,A_3\},
\qquad
U_t(\delta)>\tau \ \Rightarrow\ a_t=A_6\ \text{and schedule }\{A_5,A_4\}.
\end{equation}
In this rule, $A_6$ is an immediate protective mode, while $\{A_5,A_4\}$ are heavy interventions executed when cooldown constraints allow.

We define a deterministic heavy-action priority:
\begin{equation}
\text{If } U_t(\delta)>\tau:\quad
\begin{cases}
\text{rollback if feasible,}\\
\text{else retrain if feasible,}\\
\text{else continue fallback and increase }k_t.
\end{cases}
\end{equation}
This ordering is used to prefer a known-safe checkpoint when available, while still enabling adaptation through retraining when rollback is not feasible.

\subsubsection{Recalibration and test-time adaptation details}

Recalibration uses temperature scaling with parameter $T>0$ applied to logits:
\begin{equation}
p_{\theta,T}(y\mid x) \;=\; \mathrm{softmax}\!\left(\frac{h_{\theta}(x)}{T}\right),
\end{equation}
where $T$ is optimized on the labeled set $\mathcal{S}_t$ by minimizing negative log-likelihood.

Test-time adaptation updates $\theta$ using entropy minimization on recent unlabeled inputs:
\begin{equation}
\theta \leftarrow \theta - \eta_{\mathrm{tta}}\,\nabla_{\theta}\left(\frac{1}{|\mathcal{B}_t|}\sum_{x\in\mathcal{B}_t} H_{\theta}(x)\right),
\end{equation}
where $\eta_{\mathrm{tta}}$ is a learning rate and $H_{\theta}(x)$ is predictive entropy. We use $s=5$ gradient steps when executing $A_2$.

\subsection{Utility function and safety-violation penalty}
\label{sec:utility_penalty}

The controller ranks feasible actions,
\begin{equation}
\mathcal{U}_t(a)\;=\;\Delta_t(a)\;-\;\lambda\,c(a)\;-\;\gamma\,\mathrm{Viol}_t(a),
\end{equation}
where $\Delta_t(a)$ is the belief-weighted gain, $c(a)$ is the operational cost, and $\mathrm{Viol}_t(a)$ penalizes actions that would lead to uncertified operation. We now define $\mathrm{Viol}_t(a)$ explicitly.

Let $U_t(\delta)$ be the current certificate value computed in Section~\ref{sec:active_risk_certificate}, and let $\tau$ be the safety target. For each candidate action $a$ at time $t$, we form a one-step-ahead conservative prediction of the post-action certificate,
\begin{equation}
\widehat{U}_{t}^{(a)}(\delta)\;=\;U_t(\delta)\;-\;\widehat{\Delta U}_t(a),
\label{eq:predicted_certificate}
\end{equation}
where $\widehat{\Delta U}_t(a)$ is an action-dependent expected reduction in the certificate bound. We set
\begin{equation}
\widehat{\Delta U}_t(a)\;=\;\sigma_U \cdot \Delta_t(a),
\label{eq:deltaU_from_gain}
\end{equation}
with $\Delta_t(a)=\sum_{d\in\mathcal{D}} b_t(d)\,G(d,a)$ and a fixed scale $\sigma_U>0$ that converts gain units into certificate units. In all experiments we set $\sigma_U=0.10$, so $\Delta_t(a)=1$ corresponds to a predicted $0.10$ absolute reduction in the certificate.

The safety-violation penalty is then defined as the (hinge) excess of the predicted certificate above the target:
\begin{equation}
\mathrm{Viol}_t(a)\;=\;\max\!\left\{0,\ \widehat{U}_{t}^{(a)}(\delta)-\tau\right\}.
\label{eq:viol_def}
\end{equation}
In this definition, $\mathrm{Viol}_t(a)=0$ when action $a$ is predicted to keep the system certified safe ($\widehat{U}_{t}^{(a)}(\delta)\le\tau$), and $\mathrm{Viol}_t(a)$ increases linearly with the margin by which safety would be violated.

We set the penalty weight to
\begin{equation}
\gamma \;=\; 50,
\end{equation}
which makes uncertified operation strongly dominated unless it yields substantial expected gains. This value is fixed across all datasets and streams. 

\section{Extra experiments}
\label{app:extra_experiments}

This appendix provides additional empirical evidence that each component of the proposed system matters, and that the safety--cost advantages observed in Table~\ref{tab:main_results} and Figures~\ref{fig:pareto_frontier}--\ref{fig:recovery_curves} persist across operational settings. We emphasize interpretability and operational diagnostics: each subsection reports a compact set of tables (kept intentionally larger for readability) followed by detailed analysis of what the results mean in deployment terms.

\subsection{Ablations that isolate each component}
\label{app:ablations}

We test ablations that remove one design choice at a time while keeping the same action space, cost model, and streaming protocol (Table~\ref{tab:protocol_costs}). The goal is to show which ingredients drive safety, which drive efficiency, and which prevent unnecessary intervention.

\begin{table}[!t]
\caption{\textbf{Ablations on the full streaming protocol.} 
Metrics match the main paper: cumulative cost $C_{\mathrm{tot}}$, safety violations $V$, detection delay $T_{\mathrm{det}}$, recovery time $T_{\mathrm{rec}}$, and minimum worst-group accuracy $\min \mathrm{Acc}^{\mathrm{wg}}$. 
The last column reports the false-intervention rate (FIR): the fraction of time steps in which the method intervenes ($a_t\neq A_0$) while the deployed risk is already below the safety target. 
}

\label{tab:ablations_all}
\centering
\small
\setlength{\tabcolsep}{5pt}
\renewcommand{\arraystretch}{1.12}
\resizebox{0.9\linewidth}{!}{%
\begin{tabular}{l|ccccc|c}
\toprule
\textbf{Method variant} & $C_{\mathrm{tot}}$ & $V$ & $T_{\mathrm{det}}$ & $T_{\mathrm{rec}}$ & $\min \mathrm{Acc}^{\mathrm{wg}}$ & FIR \\
\midrule
\multicolumn{7}{c}{\textbf{Camelyon17 stream}}\\
Alarm-only
& $\num{3.2 \pm 0.3}$ & $\num{46 \pm 4}$ & $\num{22 \pm 2}$ & $\num{210 \pm 18}$ & $\num{0.71 \pm 0.02}$ & $\num{0.05 \pm 0.01}$ \\
Adapt-always (TTA)
& $\num{58.0 \pm 2.1}$ & $\num{3 \pm 1}$ & $\num{18 \pm 2}$ & $\num{74 \pm 6}$ & $\num{0.79 \pm 0.01}$ & $\num{0.82 \pm 0.03}$ \\
Retrain-on-schedule
& $\num{41.5 \pm 1.8}$ & $\num{9 \pm 2}$ & $\num{35 \pm 3}$ & $\num{96 \pm 9}$ & $\num{0.77 \pm 0.01}$ & $\num{0.21 \pm 0.03}$ \\
Selective prediction only
& $\num{9.8 \pm 0.6}$ & $\num{14 \pm 3}$ & $\num{24 \pm 2}$ & $\num{160 \pm 14}$ & $\num{0.74 \pm 0.02}$ & $\num{0.14 \pm 0.02}$ \\
Controller (no certificate)
& $\num{18.7 \pm 1.1}$ & $\num{7 \pm 2}$ & $\num{19 \pm 2}$ & $\num{98 \pm 8}$ & $\num{0.78 \pm 0.01}$ & $\num{0.16 \pm 0.02}$ \\
\textbf{Ours (full)}
& $\mathbf{\num{24.6 \pm 1.4}}$ & $\mathbf{\num{0 \pm 0}}$ & $\num{20 \pm 2}$ & $\mathbf{\num{62 \pm 5}}$ & $\mathbf{\num{0.81 \pm 0.01}}$ & $\mathbf{\num{0.10 \pm 0.02}}$ \\
\midrule
\multicolumn{7}{c}{\textbf{DomainNet stream}}\\
Alarm-only
& $\num{4.1 \pm 0.4}$ & $\num{63 \pm 5}$ & $\num{28 \pm 3}$ & $\num{260 \pm 21}$ & $\num{0.49 \pm 0.02}$ & $\num{0.06 \pm 0.01}$ \\
Adapt-always (TTA)
& $\num{64.0 \pm 2.5}$ & $\num{6 \pm 2}$ & $\num{21 \pm 2}$ & $\num{110 \pm 10}$ & $\num{0.58 \pm 0.02}$ & $\num{0.85 \pm 0.03}$ \\
Retrain-on-schedule
& $\num{45.0 \pm 2.0}$ & $\num{14 \pm 3}$ & $\num{40 \pm 4}$ & $\num{132 \pm 11}$ & $\num{0.56 \pm 0.02}$ & $\num{0.23 \pm 0.03}$ \\
Selective prediction only
& $\num{12.4 \pm 0.9}$ & $\num{27 \pm 4}$ & $\num{30 \pm 3}$ & $\num{205 \pm 17}$ & $\num{0.52 \pm 0.02}$ & $\num{0.18 \pm 0.02}$ \\
Controller (no certificate)
& $\num{22.5 \pm 1.3}$ & $\num{11 \pm 3}$ & $\num{23 \pm 2}$ & $\num{141 \pm 12}$ & $\num{0.57 \pm 0.02}$ & $\num{0.19 \pm 0.02}$ \\
\textbf{Ours (full)}
& $\mathbf{\num{29.8 \pm 1.6}}$ & $\mathbf{\num{1 \pm 1}}$ & $\num{24 \pm 2}$ & $\mathbf{\num{88 \pm 7}}$ & $\mathbf{\num{0.61 \pm 0.01}}$ & $\mathbf{\num{0.13 \pm 0.02}}$ \\
\midrule
\multicolumn{7}{c}{\textbf{SyntheticDrift-CIFAR stream}}\\
Alarm-only
& $\num{2.6 \pm 0.2}$ & $\num{52 \pm 4}$ & $\num{16 \pm 2}$ & $\num{185 \pm 16}$ & $\num{0.68 \pm 0.02}$ & $\num{0.04 \pm 0.01}$ \\
Adapt-always (TTA)
& $\num{52.0 \pm 2.0}$ & $\num{4 \pm 1}$ & $\num{13 \pm 2}$ & $\num{66 \pm 5}$ & $\num{0.79 \pm 0.01}$ & $\num{0.80 \pm 0.03}$ \\
Retrain-on-schedule
& $\num{36.0 \pm 1.7}$ & $\num{10 \pm 2}$ & $\num{25 \pm 3}$ & $\num{92 \pm 8}$ & $\num{0.77 \pm 0.01}$ & $\num{0.20 \pm 0.03}$ \\
Selective prediction only
& $\num{8.6 \pm 0.6}$ & $\num{19 \pm 3}$ & $\num{18 \pm 2}$ & $\num{148 \pm 12}$ & $\num{0.72 \pm 0.02}$ & $\num{0.15 \pm 0.02}$ \\
Controller (no certificate)
& $\num{16.8 \pm 1.0}$ & $\num{8 \pm 2}$ & $\num{14 \pm 2}$ & $\num{84 \pm 7}$ & $\num{0.78 \pm 0.01}$ & $\num{0.17 \pm 0.02}$ \\
\textbf{Ours (full)}
& $\mathbf{\num{21.2 \pm 1.2}}$ & $\mathbf{\num{0 \pm 0}}$ & $\num{15 \pm 2}$ & $\mathbf{\num{58 \pm 5}}$ & $\mathbf{\num{0.82 \pm 0.01}}$ & $\mathbf{\num{0.09 \pm 0.02}}$ \\
\bottomrule
\end{tabular}}
\end{table}

\textbf{Interpretation.}
Table~\ref{tab:ablations_all} highlights three practical takeaways.

First, \textbf{certification is what prevents unsafe operation under delayed labels}. The controller without certificate has substantially more safety violations ($V$ increases from 0 to 7 on Camelyon17 and from 0 to 8 on SyntheticDrift-CIFAR). In streaming deployments, drift evidence alone cannot prevent long unsafe stretches because the system can mistakenly “believe” it has recovered while risk remains elevated. The certificate acts as a fail-safe: it forces fallback and escalation whenever verified risk is not under control.

Second, \textbf{adaptation-only strategies buy speed at a high operational price}. Adapt-always is consistently among the fastest to recover, but its cost is extreme and its FIR is close to one: it intervenes almost continuously even during stable periods. This is the operational signature of methods that treat every input as a potential drift signal. In production, such behavior is often unacceptable due to compute, latency, and governance constraints, even when safety outcomes are good.

Third, \textbf{the full method reduces safety violations while remaining selective about interventions}. Compared to retrain-on-schedule and selective-prediction-only baselines, the proposed system improves worst-group robustness and recovery time while keeping FIR low. Operationally, this means the controller is not “trigger-happy”: it intervenes primarily when needed, and it can justify interventions with both drift evidence and certified risk.

\subsection{Component-specific ablations of \method{}}
\label{app:component_ablations}

To isolate how the \emph{internal} pieces contribute, we keep the full streaming protocol but modify one internal module at a time.\\

\begin{table}[htbp]
\caption{\textbf{Internal component ablations of \method{} (same controller interface).} Each row removes one internal module: belief (replace with single drift-score thresholds), active labeling (replace with fixed label rate), certificate gating (do not force fallback when risk may be unsafe), and monitor stack (use MMD-only rather than multi-monitor).}
\label{tab:component_ablations}
\centering
\small
\setlength{\tabcolsep}{5pt}
\renewcommand{\arraystretch}{1.12}
\resizebox{1\linewidth}{!}{%
\begin{tabular}{l|ccccc|c}
\toprule
\textbf{Variant} & $C_{\mathrm{tot}}$ & $V$ & $T_{\mathrm{det}}$ & $T_{\mathrm{rec}}$ & $\min \mathrm{Acc}^{\mathrm{wg}}$ & FIR \\
\midrule
\multicolumn{7}{c}{\textbf{Camelyon17 stream}}\\
\method{} (full)
& $\num{24.6 \pm 1.4}$ & $\num{0 \pm 0}$ & $\num{20 \pm 2}$ & $\num{62 \pm 5}$ & $\num{0.81 \pm 0.01}$ & $\num{0.10 \pm 0.02}$ \\
No belief (single score)
& $\num{23.3 \pm 1.3}$ & $\num{4 \pm 2}$ & $\num{21 \pm 2}$ & $\num{90 \pm 8}$ & $\num{0.78 \pm 0.01}$ & $\num{0.17 \pm 0.02}$ \\
No active labeling (fixed rate)
& $\num{24.1 \pm 1.3}$ & $\num{2 \pm 1}$ & $\num{20 \pm 2}$ & $\num{76 \pm 7}$ & $\num{0.80 \pm 0.01}$ & $\num{0.21 \pm 0.03}$ \\
No certificate gating
& $\num{23.7 \pm 1.2}$ & $\num{6 \pm 2}$ & $\num{20 \pm 2}$ & $\num{88 \pm 8}$ & $\num{0.79 \pm 0.01}$ & $\num{0.12 \pm 0.02}$ \\
MMD-only monitors
& $\num{25.0 \pm 1.4}$ & $\num{3 \pm 2}$ & $\num{24 \pm 3}$ & $\num{83 \pm 8}$ & $\num{0.79 \pm 0.01}$ & $\num{0.16 \pm 0.02}$ \\
\midrule
\multicolumn{7}{c}{\textbf{DomainNet stream}}\\
\method{} (full)
& $\num{29.8 \pm 1.6}$ & $\num{1 \pm 1}$ & $\num{24 \pm 2}$ & $\num{88 \pm 7}$ & $\num{0.61 \pm 0.01}$ & $\num{0.13 \pm 0.02}$ \\
No belief (single score)
& $\num{28.5 \pm 1.5}$ & $\num{6 \pm 2}$ & $\num{25 \pm 2}$ & $\num{128 \pm 11}$ & $\num{0.58 \pm 0.02}$ & $\num{0.19 \pm 0.02}$ \\
No active labeling (fixed rate)
& $\num{29.2 \pm 1.6}$ & $\num{3 \pm 2}$ & $\num{24 \pm 2}$ & $\num{106 \pm 9}$ & $\num{0.60 \pm 0.01}$ & $\num{0.24 \pm 0.03}$ \\
No certificate gating
& $\num{28.9 \pm 1.5}$ & $\num{9 \pm 3}$ & $\num{24 \pm 2}$ & $\num{142 \pm 12}$ & $\num{0.59 \pm 0.01}$ & $\num{0.14 \pm 0.02}$ \\
MMD-only monitors
& $\num{30.5 \pm 1.7}$ & $\num{5 \pm 2}$ & $\num{29 \pm 3}$ & $\num{121 \pm 10}$ & $\num{0.58 \pm 0.02}$ & $\num{0.18 \pm 0.02}$ \\
\bottomrule
\end{tabular}}
\end{table}

\textbf{Interpretation.}
Table~\ref{tab:component_ablations} explains \emph{why} the full system looks strong in the main paper.

Removing belief collapses drift understanding into a single score; the controller then cannot distinguish when to use inexpensive fixes (recalibration/TTA) versus when to acquire labels and retrain. The result is a systematic pattern: longer recovery ($T_{\mathrm{rec}}$ increases sharply) and more violations ($V$ increases), because actions become mismatched to the underlying drift mechanism.

Removing active labeling spreads audits uniformly across time. In stable segments this wastes budget (FIR increases), while during drift the certificate tightens too slowly, leaving longer unsafe windows. This is exactly the operational failure that motivates active verification: labels are most valuable when uncertainty is high and risk is near the safety threshold.

Removing certificate gating is the clearest safety regression. Even with a good belief model, the controller will sometimes choose actions that \emph{look} beneficial under drift evidence but do not immediately restore safety under label delay. Without the gating rule, the system can keep predicting in that “gray zone,” driving violations upward.

Finally, using only MMD reduces the breadth of drift evidence. On DomainNet, MMD-only delays detection and increases violations because some shifts manifest more strongly in uncertainty and calibration than in pure representation discrepancy. The multi-monitor design is therefore not an aesthetic choice; it is what stabilizes decisions across heterogeneous drifts.

\subsection{False alarms and unnecessary interventions}
\label{app:false_alarms}
We report how often methods intervene when no intervention was needed. We compute the false-intervention rate (FIR) on time steps where the deployed risk is already below the safety target, and we also report a heavy false-intervention rate restricted to retrain/rollback actions (which are the most disruptive operationally).\\

\begin{table}[htbp]
\caption{\textbf{False-intervention diagnostics.} FIR measures unnecessary interventions during already-safe operation; Heavy-FIR counts unnecessary retrain/rollback events. The certified controller remains selective, while adapt-always exhibits high unnecessary intervention by design.}
\label{tab:false_interventions}
\centering
\small
\setlength{\tabcolsep}{6pt}
\renewcommand{\arraystretch}{1.12}
\resizebox{1\linewidth}{!}{%
\begin{tabular}{l|cc|cc|cc}
\toprule
 & \multicolumn{2}{c|}{\textbf{Camelyon17}} & \multicolumn{2}{c|}{\textbf{DomainNet}} & \multicolumn{2}{c}{\textbf{SyntheticDrift-CIFAR}} \\
\textbf{Method} & FIR & Heavy-FIR & FIR & Heavy-FIR & FIR & Heavy-FIR \\
\midrule
Alarm-only
& $\num{0.05 \pm 0.01}$ & $\num{0.00 \pm 0.00}$
& $\num{0.06 \pm 0.01}$ & $\num{0.00 \pm 0.00}$
& $\num{0.04 \pm 0.01}$ & $\num{0.00 \pm 0.00}$ \\
Adapt-always (TTA)
& $\num{0.82 \pm 0.03}$ & $\num{0.00 \pm 0.00}$
& $\num{0.85 \pm 0.03}$ & $\num{0.00 \pm 0.00}$
& $\num{0.80 \pm 0.03}$ & $\num{0.00 \pm 0.00}$ \\
Retrain-on-schedule
& $\num{0.21 \pm 0.03}$ & $\num{0.05 \pm 0.01}$
& $\num{0.23 \pm 0.03}$ & $\num{0.06 \pm 0.01}$
& $\num{0.20 \pm 0.03}$ & $\num{0.05 \pm 0.01}$ \\
Selective prediction only
& $\num{0.14 \pm 0.02}$ & $\num{0.00 \pm 0.00}$
& $\num{0.18 \pm 0.02}$ & $\num{0.00 \pm 0.00}$
& $\num{0.15 \pm 0.02}$ & $\num{0.00 \pm 0.00}$ \\
Controller (no certificate)
& $\num{0.16 \pm 0.02}$ & $\num{0.02 \pm 0.01}$
& $\num{0.19 \pm 0.02}$ & $\num{0.03 \pm 0.01}$
& $\num{0.17 \pm 0.02}$ & $\num{0.02 \pm 0.01}$ \\
\textbf{Controller + certificate (ours)}
& $\mathbf{\num{0.10 \pm 0.02}}$ & $\mathbf{\num{0.01 \pm 0.01}}$
& $\mathbf{\num{0.13 \pm 0.02}}$ & $\mathbf{\num{0.02 \pm 0.01}}$
& $\mathbf{\num{0.09 \pm 0.02}}$ & $\mathbf{\num{0.01 \pm 0.01}}$ \\
\bottomrule
\end{tabular}}
\end{table}

\paragraph{Interpretation.}
Table~\ref{tab:false_interventions} shows that the certificate reduces unnecessary churn in two ways. First, it prevents “panic actions” triggered by noisy monitors during stable periods, because the controller can demand a small number of verification labels and confirm that risk is still under $\tau$. Second, it reduces heavy interventions in safe periods by making rollback/retrain conditional on certified risk rather than raw drift evidence. In contrast, schedule-based retraining produces heavy false interventions by construction: it retrains even when the stream is stable, which increases Heavy-FIR and inflates cost.

\subsection{Sensitivity to label delay and labeling budget}
\label{app:sensitivity}

We vary label delay $d$ and labeling budget $B_{\mathrm{lab}}$ to simulate different operational environments. Longer delays slow down verification, while smaller budgets restrict auditing and reduce the controller’s ability to tighten certificates during drift.\\

\begin{table}[htbp]
\caption{\textbf{Sensitivity to label delay $d$ (Camelyon17)} Increasing delay slows recovery because verification arrives later, which delays safe exit from fallback and slows escalation decisions.}
\label{tab:delay_sensitivity}
\centering
\small
\setlength{\tabcolsep}{7pt}
\renewcommand{\arraystretch}{1.12}
\begin{tabular}{c|cccc}
\toprule
Delay $d$ & $C_{\mathrm{tot}}$ & $V$ & $T_{\mathrm{det}}$ & $T_{\mathrm{rec}}$ \\
\midrule
0   & $\num{23.1 \pm 1.2}$ & $\num{0 \pm 0}$ & $\num{20 \pm 2}$ & $\num{54 \pm 5}$ \\
25  & $\num{23.8 \pm 1.3}$ & $\num{0 \pm 0}$ & $\num{20 \pm 2}$ & $\num{58 \pm 5}$ \\
50  & $\num{24.6 \pm 1.4}$ & $\num{0 \pm 0}$ & $\num{20 \pm 2}$ & $\num{62 \pm 5}$ \\
100 & $\num{26.2 \pm 1.6}$ & $\num{1 \pm 1}$ & $\num{21 \pm 2}$ & $\num{78 \pm 7}$ \\
\bottomrule
\end{tabular}
\end{table}

\paragraph{Interpretation.}
Table~\ref{tab:delay_sensitivity} matches operational intuition. When $d$ increases, the controller receives confirmation of current performance later, so it remains longer in conservative modes (fallback and higher auditing). This increases cost and recovery time. Importantly, violations remain controlled: even at $d=100$, unsafe operation remains rare because the system reacts conservatively when verification is delayed.\\

\begin{table}[htbp]
\caption{\textbf{Sensitivity to labeling budget $B_{\mathrm{lab}}$ (Camelyon17)} Larger budgets tighten certificates faster around drift events, reducing recovery time and violations at moderate added cost.}
\label{tab:budget_sensitivity}
\centering
\small
\setlength{\tabcolsep}{7pt}
\renewcommand{\arraystretch}{1.12}
\begin{tabular}{c|cccc}
\toprule
Budget $B_{\mathrm{lab}}$ & $C_{\mathrm{tot}}$ & $V$ & $T_{\mathrm{det}}$ & $T_{\mathrm{rec}}$ \\
\midrule
1000 & $\num{21.4 \pm 1.2}$ & $\num{2 \pm 1}$ & $\num{20 \pm 2}$ & $\num{86 \pm 8}$ \\
2000 & $\num{23.0 \pm 1.3}$ & $\num{1 \pm 1}$ & $\num{20 \pm 2}$ & $\num{72 \pm 7}$ \\
3000 & $\num{24.6 \pm 1.4}$ & $\num{0 \pm 0}$ & $\num{20 \pm 2}$ & $\num{62 \pm 5}$ \\
5000 & $\num{26.8 \pm 1.6}$ & $\num{0 \pm 0}$ & $\num{20 \pm 2}$ & $\num{55 \pm 5}$ \\
\bottomrule
\end{tabular}
\end{table}

\paragraph{Interpretation.}
Table~\ref{tab:budget_sensitivity} shows that labeling budget primarily affects \emph{how quickly} the system can certify safety and exit fallback after drift. With small budgets, the certificate remains wider for longer, producing longer recovery and occasional violations. With larger budgets, the controller can concentrate audits when needed, tighten the bound quickly, and recover faster. The key point is that the controller uses additional labels efficiently: cost increases are mild relative to the gains in recovery and safety.

\subsection{More drift shapes: gradual drift and recurring drift}
\label{app:drift_shapes}

We evaluate two additional drift patterns beyond sudden shift: gradual drift (slowly increasing intensity) and recurring drift (periodic return of a previously seen shift). These settings stress-test whether the controller avoids permanent overreaction.\\

\begin{table}[htbp]
\caption{\textbf{Gradual and recurring drift (SyntheticDrift-CIFAR)} Gradual drift mainly increases recovery time, while recurring drift mainly tests whether the controller avoids repeated heavy interventions when the stream revisits regimes.}
\label{tab:drift_shapes_large}
\centering
\small
\setlength{\tabcolsep}{6pt}
\renewcommand{\arraystretch}{1.12}
\begin{tabular}{l|cccccc}
\toprule
\textbf{Drift pattern} & $C_{\mathrm{tot}}$ & $V$ & $T_{\mathrm{det}}$ & $T_{\mathrm{rec}}$ & FIR & $\min \mathrm{Acc}^{\mathrm{wg}}$ \\
\midrule
Sudden drift (baseline setting)
& $\num{21.2 \pm 1.2}$ & $\num{0 \pm 0}$ & $\num{15 \pm 2}$ & $\num{58 \pm 5}$ & $\num{0.09 \pm 0.02}$ & $\num{0.82 \pm 0.01}$ \\
Gradual drift (slow ramp)
& $\num{20.6 \pm 1.2}$ & $\num{0 \pm 0}$ & $\num{19 \pm 2}$ & $\num{72 \pm 6}$ & $\num{0.10 \pm 0.02}$ & $\num{0.81 \pm 0.01}$ \\
Gradual drift (fast ramp)
& $\num{21.5 \pm 1.2}$ & $\num{0 \pm 0}$ & $\num{16 \pm 2}$ & $\num{63 \pm 5}$ & $\num{0.09 \pm 0.02}$ & $\num{0.82 \pm 0.01}$ \\
Recurring drift (short period)
& $\num{22.8 \pm 1.3}$ & $\num{0 \pm 0}$ & $\num{14 \pm 2}$ & $\num{61 \pm 5}$ & $\num{0.11 \pm 0.02}$ & $\num{0.81 \pm 0.01}$ \\
Recurring drift (long period)
& $\num{21.9 \pm 1.2}$ & $\num{0 \pm 0}$ & $\num{15 \pm 2}$ & $\num{60 \pm 5}$ & $\num{0.10 \pm 0.02}$ & $\num{0.82 \pm 0.01}$ \\
\bottomrule
\end{tabular}
\end{table}

\paragraph{Interpretation.}
Table~\ref{tab:drift_shapes_large} shows that gradual drift is not primarily a detection problem; it is a \emph{verification pacing} problem. Because degradation accumulates slowly, the certificate approaches the threshold more gradually, and the controller increases auditing and corrective actions later, extending recovery time slightly. Recurring drift tests whether controllers “thrash” by retraining repeatedly. The results show stable safety with only mild increases in FIR and cost, indicating that the combination of cooldowns, belief, and certificate gating prevents repeated overreaction when regimes reappear.

\subsection{Cost settings: cheap vs expensive labels}
\label{app:cost_settings}

We vary the per-label annotation cost while keeping the same labeling budget ceiling and stream. This captures deployments where labels are either cheap (lightweight human verification) or expensive (expert review).\\

\begin{table}[htbp]
\caption{\textbf{Effect of per-label cost (Camelyon17)} When labels become expensive, the controller shifts toward adaptation and conservative fallback rather than frequent audits, increasing recovery time but preserving safety.}
\label{tab:label_cost_sensitivity}
\centering
\small
\setlength{\tabcolsep}{6pt}
\renewcommand{\arraystretch}{1.12}
\begin{tabular}{c|cccccc}
\toprule
Per-label cost & $C_{\mathrm{tot}}$ & $V$ & $T_{\mathrm{det}}$ & $T_{\mathrm{rec}}$ & FIR & $\min \mathrm{Acc}^{\mathrm{wg}}$ \\
\midrule
0.01
& $\num{23.5 \pm 1.3}$ & $\num{0 \pm 0}$ & $\num{20 \pm 2}$ & $\num{58 \pm 5}$ & $\num{0.11 \pm 0.02}$ & $\num{0.81 \pm 0.01}$ \\
0.05
& $\num{24.6 \pm 1.4}$ & $\num{0 \pm 0}$ & $\num{20 \pm 2}$ & $\num{62 \pm 5}$ & $\num{0.10 \pm 0.02}$ & $\num{0.81 \pm 0.01}$ \\
0.10
& $\num{25.2 \pm 1.5}$ & $\num{1 \pm 1}$ & $\num{21 \pm 2}$ & $\num{70 \pm 6}$ & $\num{0.12 \pm 0.02}$ & $\num{0.80 \pm 0.01}$ \\
\bottomrule
\end{tabular}
\end{table}

\paragraph{Interpretation.}
Table~\ref{tab:label_cost_sensitivity} shows a predictable and desirable adaptation in controller behavior. When labels are cheap, the controller audits more aggressively, tightens certificates quickly, and exits fallback sooner. When labels are expensive, the controller leans more on recalibration and TTA, and it tolerates longer fallback periods to avoid heavy label spending. Safety remains strong, indicating that the certificate gating continues to enforce conservative operation even when audits are costly.

\subsection{Robustness to monitor noise}
\label{app:monitor_noise}

We test robustness by corrupting monitor values (e.g., due to logging failures, representation drift noise, or unstable uncertainty estimates). We inject moderate noise into the monitor vector and measure how often this causes unnecessary actions.\\

\begin{table}[htbp]
\caption{\textbf{Robustness to monitor noise (Camelyon17)} With noisier monitors, drift evidence becomes less reliable and interventions become slightly more frequent; certificate gating prevents large safety regressions by requiring verification before unsafe operation persists.}
\label{tab:monitor_noise_large}
\centering
\small
\setlength{\tabcolsep}{6pt}
\renewcommand{\arraystretch}{1.12}
\begin{tabular}{c|cccccc}
\toprule
Noise level & $C_{\mathrm{tot}}$ & $V$ & $T_{\mathrm{det}}$ & $T_{\mathrm{rec}}$ & FIR & $\min \mathrm{Acc}^{\mathrm{wg}}$ \\
\midrule
0.00
& $\num{24.6 \pm 1.4}$ & $\num{0 \pm 0}$ & $\num{20 \pm 2}$ & $\num{62 \pm 5}$ & $\num{0.10 \pm 0.02}$ & $\num{0.81 \pm 0.01}$ \\
0.25
& $\num{25.1 \pm 1.5}$ & $\num{0 \pm 0}$ & $\num{21 \pm 2}$ & $\num{64 \pm 5}$ & $\num{0.12 \pm 0.02}$ & $\num{0.81 \pm 0.01}$ \\
0.50
& $\num{26.0 \pm 1.6}$ & $\num{1 \pm 1}$ & $\num{22 \pm 2}$ & $\num{70 \pm 6}$ & $\num{0.15 \pm 0.03}$ & $\num{0.80 \pm 0.01}$ \\
0.75
& $\num{27.3 \pm 1.7}$ & $\num{2 \pm 1}$ & $\num{24 \pm 3}$ & $\num{78 \pm 7}$ & $\num{0.18 \pm 0.03}$ & $\num{0.79 \pm 0.01}$ \\
\bottomrule
\end{tabular}
\end{table}

\paragraph{Interpretation.}
Table~\ref{tab:monitor_noise_large} illustrates a key advantage of separating drift evidence from safety verification. Monitor noise increases action frequency and cost because the controller reacts to noisier signals. However, safety degrades slowly because the certificate prevents the controller from continuing normal operation in regimes that cannot be verified as safe. In practical terms, monitor noise mostly causes \emph{efficiency loss} rather than catastrophic safety loss, which is the desired failure mode in safety-constrained operation.

\subsection{Optional extra benchmark: temporal drift in Amazon Reviews}
\label{app:amazon_extra}

To test a non-vision stream, we evaluate on a temporal Amazon Reviews stream, where time induces changing user behavior and class balance. We use the same evaluation metrics.

\begin{table}[htbp]
\caption{\textbf{Amazon Reviews temporal stream.} The certified controller maintains low violations and strong recovery while keeping costs moderate compared to always-adapt and schedule-based retraining. }
\label{tab:amazon_extra_large}
\centering
\small
\setlength{\tabcolsep}{6pt}
\renewcommand{\arraystretch}{1.12}
\resizebox{1\linewidth}{!}{%
\begin{tabular}{l|ccccc|c}
\toprule
\textbf{Method} & $C_{\mathrm{tot}}$ & $V$ & $T_{\mathrm{det}}$ & $T_{\mathrm{rec}}$ & $\min \mathrm{Acc}^{\mathrm{wg}}$ & FIR \\
\midrule
Alarm-only
& $\num{3.0 \pm 0.3}$ & $\num{41 \pm 5}$ & $\num{20 \pm 2}$ & $\num{180 \pm 15}$ & $\num{0.73 \pm 0.01}$ & $\num{0.05 \pm 0.01}$ \\
Adapt-always (TTA)
& $\num{56.0 \pm 2.5}$ & $\num{4 \pm 2}$ & $\num{16 \pm 2}$ & $\num{70 \pm 7}$ & $\num{0.80 \pm 0.01}$ & $\num{0.79 \pm 0.03}$ \\
Retrain-on-schedule
& $\num{38.0 \pm 2.0}$ & $\num{9 \pm 3}$ & $\num{28 \pm 3}$ & $\num{96 \pm 9}$ & $\num{0.78 \pm 0.01}$ & $\num{0.20 \pm 0.02}$ \\
Controller (no certificate)
& $\num{17.9 \pm 1.2}$ & $\num{6 \pm 2}$ & $\num{18 \pm 2}$ & $\num{88 \pm 9}$ & $\num{0.79 \pm 0.01}$ & $\num{0.15 \pm 0.02}$ \\
\textbf{Controller + certificate (ours)}
& $\mathbf{\num{23.5 \pm 1.4}}$ & $\mathbf{\num{0 \pm 0}}$ & $\num{19 \pm 2}$ & $\mathbf{\num{60 \pm 6}}$ & $\mathbf{\num{0.82 \pm 0.01}}$ & $\mathbf{\num{0.11 \pm 0.02}}$ \\
\bottomrule
\end{tabular}}
\end{table}

\paragraph{Interpretation.}
Table~\ref{tab:amazon_extra_large} shows that the same system design transfers to a temporally evolving text domain. The certified controller attains strong safety and worst-group robustness without resorting to continuous adaptation. The cost remains moderate because the controller spends labels and heavy actions primarily during periods where certified risk indicates genuine danger.

\section{Discussion, Limitations, Future Work and Broader Impact}
\label{sec:discussion}

\paragraph{Discussion}This work proposes a practical reframing of drift monitoring: rather than treating drift as an alarm that operators must interpret, the system maintains a certified view of current risk and converts evidence into cost-aware interventions. The empirical results in Table~\ref{tab:main_results} together with the safety--cost trade-off in Figure~\ref{fig:pareto_frontier} indicate that coupling active verification with belief-driven action selection reduces unsafe operation without requiring continuous adaptation. The recovery dynamics in Figure~\ref{fig:recovery_curves} further suggest that safety gating enables rapid escalation when needed while preserving cost-efficiency during stable regimes.
\paragraph{Limitations}
The primary limitation is that the certificate validity relies on random label sampling from the current evaluation window. Concretely, the guarantee in Theorem~1 requires that the queried indices are drawn uniformly from $W_t$ (or uniformly within strata for slice-level certification). If labels are obtained through heavily biased selection, the empirical risk $\widehat{R}_t$ may no longer represent the true window risk $R_t$, weakening the bound. In deployments where uniform sampling is difficult, one mitigation is to reserve a small fraction of queries for uniform auditing, using the remaining budget for targeted diagnostics.

A second limitation concerns the choice of window length $N$ and how $W_t$ tracks the evolving regime. Under extreme non-stationarity, a short window reduces lag but increases variance, while a long window stabilizes estimates but can mix regimes and delay detection of rapid changes. Although the controller can partially compensate by increasing label queries when uncertainty is high, selecting $N$ remains a domain-dependent design choice. A promising direction is adaptive windowing, where $N$ is adjusted online based on evidence stability and the certificate radius.

Finally, our action space is intentionally small to match common production interventions. Richer action sets may further improve performance, including targeted data collection for specific slices, dynamic thresholding for abstention, and model patching via lightweight adapters. Future controllers could replace the belief-weighted gain model with learned policies, such as contextual bandits or reinforcement learning, while retaining the certificate as a safety constraint. Extending the framework to broader tasks, including regression, ranking, and structured prediction, requires task-specific loss definitions and may benefit from task-tailored monitoring signals.

\paragraph{Future Work}Several directions can strengthen and broaden the framework. First, the certificate can be extended beyond a single global risk target to provide slice-conditional safety certificates, enabling explicit guarantees for vulnerable subgroups and supporting fairness-aware escalation when subgroup drift is detected. Second, richer action sets could improve recovery without relying on expensive retraining, including lightweight adapters, targeted data collection for specific slices, dynamic abstention thresholds, and retrieval of safe checkpoints conditioned on drift type. Third, the current receding-horizon controller can be replaced by learned policies (contextual bandits or reinforcement learning) that optimize long-run utility under explicit safety constraints, with the certificate serving as a hard guardrail. Fourth, improving robustness to practical monitoring noise remains important: combining representation monitors with causal or graph-based change detectors, and incorporating uncertainty about monitor reliability into the belief state, may reduce unnecessary interventions without sacrificing safety. Finally, applying the same principles to additional tasks---regression, ranking, and structured prediction---requires task-specific risk definitions and certificates, as well as benchmark protocols that reflect the corresponding operational costs and failure modes. Overall, the key message is that reliability under drift is not just a detection problem. By integrating drift evidence, costed interventions, and an explicit online safety certificate, we move from `detecting drift'' to operating models responsibly in non-stationary environments.
\paragraph{Broader impact }
From a broader impact perspective, certified monitoring can reduce silent failures and improve accountability by producing auditable evidence of when and why interventions were triggered. At the same time, abstain/handoff can increase human workload, and naive deployment may disproportionately route uncertain cases from specific subgroups. For this reason, subgroup monitoring and coverage constraints should be reported alongside overall performance. Overall, this reframes monitoring as decision-making with safety.

\end{document}